\documentclass{article}
\usepackage[main, final]{neurips_2026}
\usepackage[utf8]{inputenc}
\usepackage[T1]{fontenc}
\usepackage{url}
\usepackage{booktabs}
\usepackage{amsfonts}
\usepackage{nicefrac}
\usepackage{microtype}
\usepackage{xcolor}
\usepackage{pifont}
\usepackage{amsmath}
\usepackage{amssymb}
\usepackage{amsthm}
\usepackage{wrapfig}
\usepackage{graphicx}
\usepackage{array}
\usepackage{float}
\usepackage{placeins}
\usepackage{xcolor}         
\usepackage{pifont}
\usepackage{paralist}
\usepackage{amsmath}
\usepackage{amssymb}
\usepackage{algpseudocode}
\usepackage{algorithm}
\usepackage{hyperref}   
\usepackage{amsthm}
\newtheorem{theorem}{Theorem}

\usepackage{wrapfig}
\usepackage{graphicx}
\usepackage[table]{xcolor}

\newtheorem{lemma}{Lemma}

\newtheorem{definition}{Definition}

\newcolumntype{C}{>{\centering\arraybackslash}c}

\newcommand{\Mc}{\mathcal{M}_c}
\newcommand{\Xc}{\mathcal{X}}
\newcommand{\abar}{\bar\alpha}
\newcommand{\trS}{\operatorname{tr}\Sigma_c}
\usepackage[skip=3.0pt plus 1pt]{parskip}
\edef\restorefloatsep{\textfloatsep=\the\textfloatsep\relax
  \floatsep=\the\floatsep\relax \intextsep=\the\intextsep\relax}
\title{ManifoldCache: Training-Free Diffusion Acceleration via Constraint Manifold Caching}

\author{%
  Prashant Pandey\thanks{Equal contribution.}  \hspace{0.001cm}       
  \thanks{Correspondence to: \texttt{getprashant57@gmail.com}} \\ 
  Indian Institute of Technology, Delhi \\
  \And 
  Devineni Sri Venkatraya Chowdary\footnotemark[1] \\ 
  Indian School of Mines, Dhanbad \\
  \And 
  Brejesh Lall \\ 
  Indian Institute of Technology, Delhi \\
}

\begin{document}
\maketitle


\begin{abstract}
Diffusion models for structured scientific generation must produce
samples satisfying hard geometric constraints imposed by physics,
chemistry, or biology, yet inference in these settings is
prohibitively slow, demanding hundreds to thousands of neural-function
evaluations per sample.
We unify eight state-of-the-art models spanning medical volumetrics,
molecular conformations, protein backbone design, crystal structure
prediction, and multi-view 3D scenes under a single abstraction,
\emph{Constraint-Manifold Diffusion Models} (CMDMs), in
which the target distribution is supported on a manifold defined by
an externally specified constraint map.
Existing acceleration families transfer poorly to this class: quantization
exhausts memory on high-dimensional volumetric operators; pruning
breaks constraint fidelity; fast ODE solvers allow trajectories to
drift off the constraint manifold; and feature-caching heuristics are
blind to constraint geometry and can cause mode confusion in the
high-noise regime.
We introduce \textsc{ManifoldCache}, the first training-free,
data-free accelerator designed from first principles for CMDMs.
The key insight is that the conditional score decomposes orthogonally
into a normal component, which enforces constraint satisfaction, and a
tangential component, which navigates within the manifold.
Building on this structure, we show for a two-mode Gaussian model
that the noise-schedule midpoint, where signal and noise have equal
energy, separates two regimes: once signal dominates, the error of a
cached estimate is provably bounded, while before that point a
strictly positive fraction of cached estimates land in the wrong mode,
and this fraction does not vanish as the midpoint is approached.
We further prove that, under an explicit condition on the balance
between the two score components, deeper network blocks admit larger
certified cache strides, as a consequence of the score decomposition
propagating through block Jacobians.
The resulting schedule requires one integer comparison per block per
step, zero calibration data, and zero training overhead.
Across all eight CMDMs, \textsc{ManifoldCache} delivers wall-clock
speedups of $1.2$--$1.9\times$ and stays closest to full-inference
quality on each model's primary metric among the feature-caching
methods we compare; on Med-DDPM it also offers a better
quality--speed trade-off than quantization, pruning, and fast-solver
baselines. Our code is available at
\url{https://github.com/prinshul/Manifoldcache.git}.
\end{abstract}

\section{Introduction}
\label{sec:introduction}

\looseness=-1 Diffusion models have transformed generative modelling across
images~\citep{ho2020denoising,rombach2022high},
audio~\citep{kong2020diffwave}, and
video~\citep{ho2022video}.
Their most consequential frontier lies in the structured spaces of
science and engineering: volumetric anatomical
atlases~\citep{dorjsembe2024conditional}, molecular conformations and
graphs~\citep{liu2026graph,frank2025sampling,kong2025unimomo}, protein
backbones~\citep{geffner2025proteina}, periodic crystal
lattices~\citep{jiao2024space}, and multi-view 3D
scenes~\citep{yao2025sv4d,bahmani2026lyra}.
What unites these applications is the fact that their
target distributions are supported not on the full ambient space
$\Xc$, but on a \emph{constraint manifold}
$\Mc = \{x\in\Xc : g(x;c)=0\}$,
where the constraint map $g$ encodes physics, chemistry, geometry, or
biology.
We call this class \emph{Constraint-Manifold Diffusion
Models} (\textbf{CMDMs}).
CMDMs are central to drug discovery~\citep{kong2025unimomo},
\textit{de~novo} protein design~\citep{geffner2025proteina}, and crystal
structure prediction~\citep{jiao2024space}, yet inference is
prohibitively slow: Med-DDPM~\citep{dorjsembe2024conditional} requires
{2\,min\,45\,sec} per sample on a single A100 GPU, demanding
hundreds to thousands of neural-function evaluations (NFEs).

\begin{wrapfigure}{r}{0.6\columnwidth}
  \vspace{-0.4em}
  \centering
  \includegraphics[width=0.58\columnwidth]{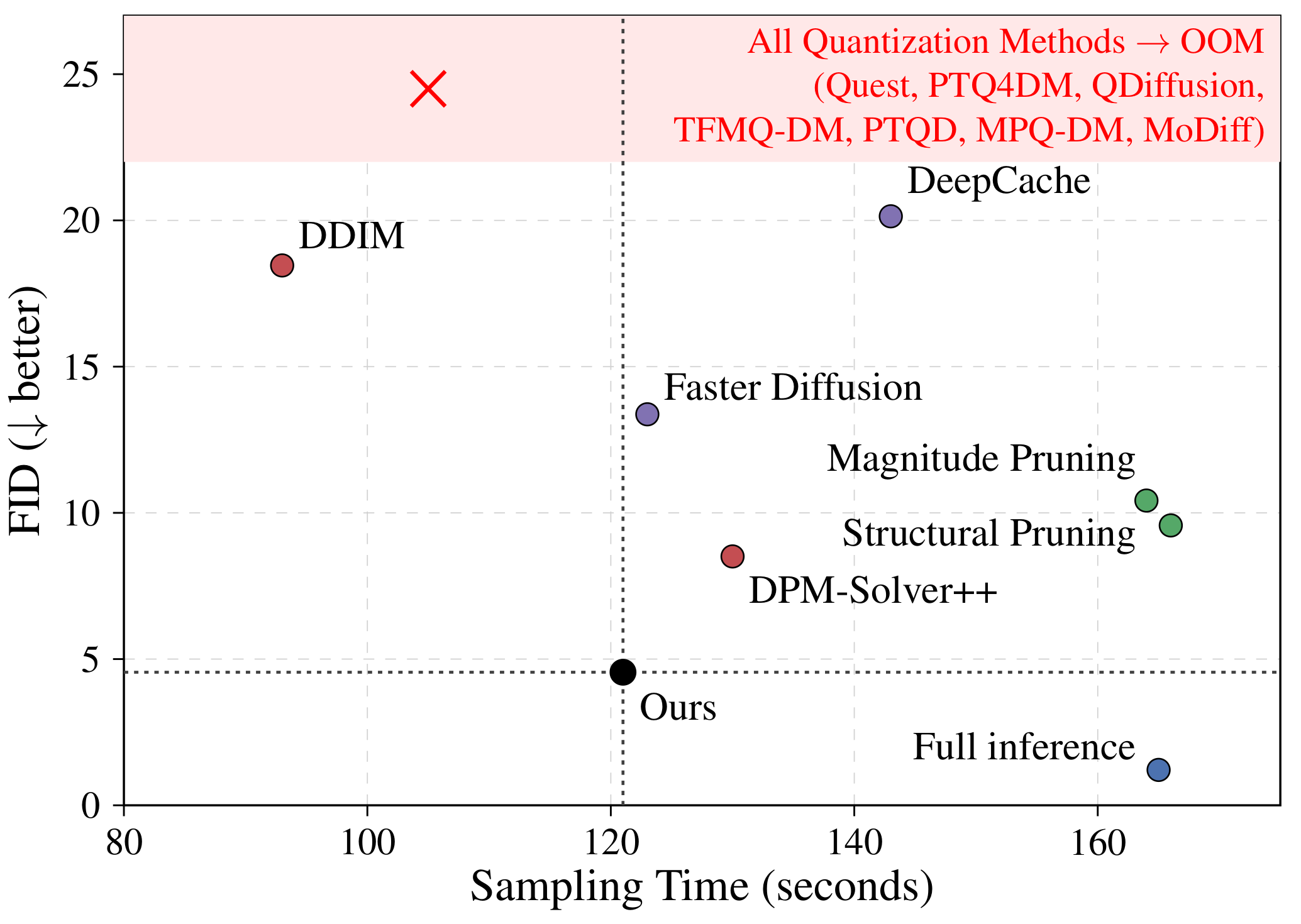}
  \caption{Quality--speed trade-off on
    Med-DDPM~\citep{dorjsembe2024conditional} across quantization, pruning,
    solver, and caching
    baselines.
    \textsc{ManifoldCache} achieves 2\,min\,1\,sec per sample (full
    inference: 2\,min\,45\,sec); all baselines either degrade quality, slow down, or incur
    OOM.}
  \label{fig:teaser}
  \vspace{-0.4em}
\end{wrapfigure}

\looseness=-1 \textbf{Why existing methods fall short on CMDMs.}
CMDMs employ fundamentally different backbone
architectures (volumetric convolutions, $\mathrm{SE}(3)$-equivariant
networks, graph transformers), whose feature dynamics are poorly
matched by heuristics designed for pixel-space UNets and vision
transformers.
\textbf{\emph{Post-training quantization}}~\citep{shang2023post,li2023q,
huang2024tfmq,he2023ptqd,feng2025mpq,gao2025modulated,wang2025quest} requires a
calibration pass that exhausts GPU memory on volumetric operators such
as the Conv3D layers in Med-DDPM.
\textbf{\emph{Pruning}}~\citep{han2015learning,fang2023structural}
removes weights that collectively enforce $g$; even 30\% sparsity
produces samples violating hard geometric constraints.
\textbf{\emph{Faster ODE solvers}}~\citep{song2020denoising,lu2025dpm} exploit
trajectory smoothness that holds only in the clean limit; large steps
allow drift off $\Mc$, and for discrete-diffusion
variants~\citep{liu2026graph} they are categorically
inapplicable.
\emph{Feature-caching methods} (DeepCache~\citep{ma2024deepcache},
SmoothCache~\citep{liu2025smoothcache},
TeaCache~\citep{liu2025timestep},
FORA~\citep{selvaraju2024fora}, and
others) exploit near-stationarity of
activations but apply cache-trigger rules blind to constraint
geometry, causing over-caching (collapsing constraint fidelity) or
under-caching (negligible speedup).
Crucially, none identifies the schedule midpoint $\abar_t=\tfrac{1}{2}$
as a safe-caching boundary, and none assigns block-level strides from
depth-dependent score structure.

\looseness=-1 We propose \textbf{\textsc{ManifoldCache}}, the first
inference-acceleration framework for CMDMs, derived from the
orthogonal decomposition of the conditional score into a normal
component $s_t^\perp$ (enforcing constraint satisfaction) and a
tangential component $s_t^\parallel$ (navigating along $\Mc$).
Theorem~\ref{thm:threshold} shows, for a two-mode Gaussian model,
that $T^*=\max\{t:\abar_t\ge\tfrac{1}{2}\}$, which depends only on the
noise schedule, separates two regimes: for $\tau\le T^*$ the expected
squared error of the frozen Tweedie estimate is at most
$\tfrac{1}{2}\trS$ and the probability of landing in the wrong mode is
small for well-separated modes, while for $\tau>T^*$ a strictly
positive fraction of frozen estimates lands in the wrong manifold mode
for every within-mode spread.
Theorem~\ref{thm:depth} establishes that within the safe phase,
deeper network blocks admit provably larger certified cache strides,
as a formal consequence of the score decomposition propagating through
block Jacobians when the normal-to-tangential power ratio $\kappa_t$
falls below an explicit threshold $\kappa^*$.
\textsc{ManifoldCache} instantiates both results as a lightweight,
training-free schedule: one integer comparison per block per step,
zero calibration, zero training overhead.
On Med-DDPM this yields ${\approx}1.4\times$ wall-clock speedup with
the lowest FID among the caching methods we compare
(Figure~\ref{fig:teaser}, Table~\ref{tab:medddpm_final}).

\looseness=-1 \textbf{Contributions.}
\textbf{(1)} A unified framework (\S\ref{sec:constraint-framework})
showing that eight state-of-the-art models (Med-DDPM~\citep{dorjsembe2024conditional},
SV4D~2.0~\citep{yao2025sv4d}, DemoDiff~\citep{liu2026graph},
Proteina~\citep{geffner2025proteina}, DiT-MC~\citep{frank2025sampling},
UniMoMo~\citep{kong2025unimomo}, DiffCSP++~\citep{jiao2024space},
Lyra~\citep{bahmani2026lyra}) are all special cases of the CMDM
construction.
\textbf{(2)} A schedule-derived safe-caching threshold $T^*$
(Theorem~\ref{thm:threshold}, \S\ref{sec:optimal-threshold})
that separates a regime of provably bounded caching error from one
with a provably non-vanishing mode-flip probability, depending only
on the noise schedule.
\textbf{(3)} Depth-monotone cacheability via Jacobian decomposition
(Theorem~\ref{thm:depth}, \S\ref{sec:depth-adaptive-schedule}),
providing the first theoretical justification for depth-adaptive
caching with an explicit regime condition.
\textbf{(4)} \textsc{ManifoldCache}
(\S\ref{sec:depth-adaptive-schedule}, Figure~\ref{fig:teaser}),
which on the eight CMDMs is faster than every baseline
while scoring best among the accelerators on the primary quality metric.

\section{Related Work}
\label{sec:related}

\looseness=-1 \textbf{Constraint-manifold diffusion models.}
A growing class of diffusion models generates samples satisfying hard
geometric, physical, or biological constraints~\citep{dorjsembe2024conditional, jiao2024space,
yao2025sv4d, liu2026graph, geffner2025proteina, frank2025sampling, kong2025unimomo}.
Their backbones (volumetric 3D ConvNets, $\mathrm{SE}(3)$-equivariant
networks, graph transformers, and multi-view ViTs) are heavier than standard 2D UNets; constraint fidelity degrades sharply
under step reduction since the probability flow is only tangent
to~$\mathcal{M}_c$ in the clean limit; and guidance strategies multiply
NFEs by 2--4$\times$.
Together, these factors significantly increase the inference and generation times for CMDMs.

\textbf{Acceleration methods.}
\textit{\underline{Post-training quantization}}~\citep{shang2023post,li2023q,
huang2024tfmq,he2023ptqd,feng2025mpq,gao2025modulated,wang2025quest} reduces per-step compute via
INT8/FP8/INT4 weight and activation compression.
On volumetric CMDMs, calibration passes exhaust GPU memory, resulting in Out of Memory (OOM) (mainly block reconstruction approaches), and quantization noise
accumulates along directions normal to~$\mathcal{M}_c$, where any
error directly manifests as constraint violation. \textit{\underline{Pruning.}}
Magnitude~\citep{han2015learning} and
structural~\citep{fang2023structural} pruning remove low-saliency
weights to reduce per-step cost.
On CMDMs every weight participates in satisfying the constraint
map~$g$; 30\% sparsity already drives generation quality below an
acceptable level, and neither variant reduces the NFE count. \textit{\underline{Fast solvers.}}
DDIM~\citep{song2020denoising} and DPM-Solver++~\citep{lu2025dpm}
reduce NFEs by taking larger probability-flow ODE integration steps,
exploiting trajectory smoothness in unconstrained image models.
On CMDMs, large steps allow trajectories to drift off~$\mathcal{M}_c$
before any corrective force restores them, producing
constraint-violating samples.
For discrete-diffusion variants such as DemoDiff~\citep{liu2026graph}, which operate on token
simplices, high-order ODE solvers are categorically inapplicable.

\textbf{Feature caching.}
DeepCache~\citep{ma2024deepcache},
FasterDiffusion~\citep{li2024faster},
FORA~\citep{selvaraju2024fora},
SmoothCache~\citep{liu2025smoothcache},
TeaCache~\citep{liu2025timestep},
AB-Cache~\citep{yu2025ab},
PFC~\citep{sommer2025predictive},
ToCa~\citep{zou2024accelerating},
dKV-Cache~\citep{ma2025dkvcache},
TaylorSeer~\citep{liu2025reusing},
PAB~\citep{zhao2408real},
AdaCache~\citep{kahatapitiya2025adaptive},
and FrDiff~\citep{so2024frdiff}
exploit the near-stationarity of hidden activations between adjacent
denoising steps via fixed intervals, calibration thresholds,
input-difference proxies, or numerical extrapolation.
These methods share a common limitation: none uses the geometric
structure of~$\mathcal{M}_c$ as a caching criterion, and none
identifies $\bar\alpha_t = \tfrac{1}{2}$ as a safe-caching boundary,
so they may reuse features in the high-noise phase where, as
Theorem~\ref{thm:threshold} shows for a two-mode model, the modes are
not yet resolved and reuse leads to mode confusion with non-vanishing
probability.
\textsc{ManifoldCache} \textit{is the first framework to derive a
depth-adaptive caching schedule with provable error bounds from
constraint geometry,
requiring no training, no calibration data, and no per-step runtime
measurement.}

\textbf{Speculative sampling.}
Speculative methods for diffusion models~\citep{debortoli2025accelerated}
and diffusion language models~\citep{agrawal2025spiffy} draft future
denoising states with a cheaper proposal and then accept or reject them
in a verification step, which reduces the number of expensive model
calls while preserving the output distribution.
\textsc{ManifoldCache} works on a different axis: it keeps every
denoising step and lowers the cost of each step by reusing block
outputs according to a fixed schedule, with no draft model or
verifier. The two approaches are therefore complementary;
Appendix~\ref{app:spec} discusses the differences in more detail and
compares them on Proteina.

\section{Methodology}
\label{sec:methodology}

\looseness=-1 \textsc{ManifoldCache} rests on a single observation: the physics or
biology that defines a constraint manifold also determines \emph{when}
and \emph{where} a denoiser's computation can be safely skipped.
We formalise this in four steps: \textbf{(i)} defining the CMDM class, \textbf{(ii)} decomposing
the score geometrically, \textbf{(iii)} deriving a universal safe-caching threshold,
and \textbf{(iv)} assembling a closed-form depth-adaptive schedule.
Figure~\ref{fig:geometry} illustrates the main ideas.

\begin{figure}[t]
  \centering
  \includegraphics[width=0.7\textwidth]{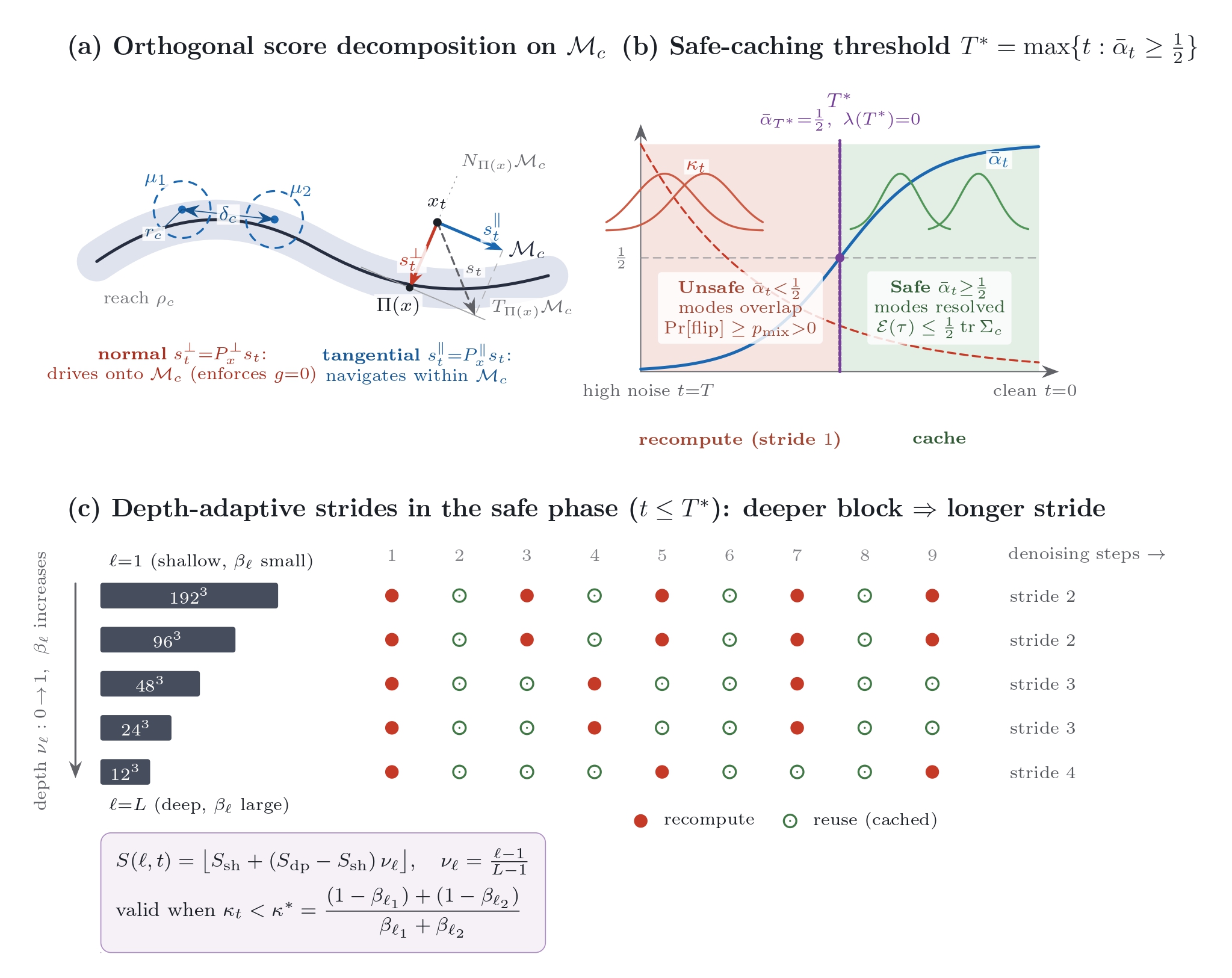}
  \caption{Geometric picture behind \textsc{ManifoldCache}.
  \textbf{(a)} Near $\Mc$, the score splits into a normal part
  $s_t^\perp$, which pulls the sample toward $\Mc$, and a tangential part
  $s_t^\parallel$, which moves it within $\Mc$ toward one of the valid
  modes (here $\mu_1$ and $\mu_2$, separated by $\delta_c$;
  \S\ref{sec:general-caching}).
  \textbf{(b)} The threshold $T^*=\max\{t:\abar_t\ge\tfrac{1}{2}\}$ splits
  the schedule: for $t>T^*$ every block is recomputed, and caching is
  enabled only for $t\le T^*$ (Theorem~\ref{thm:threshold}); the
  $\kappa_t$ curve is schematic.
  \textbf{(c)} For $t\le T^*$, the cache stride grows with block depth
  (Eq.~\eqref{eq:stride-schedule}): shallow blocks are recomputed often,
  and deep blocks are reused for longer.}
  \label{fig:geometry}
\end{figure}

\subsection{Constraint-Manifold Diffusion Models}
\label{sec:constraint-framework}
Let $\Xc\subseteq\mathbb{R}^D$ be the \emph{ambient space}, a voxel
grid, sequence of residue coordinates, periodic unit cell, or latent video
tensor depending on the application.
Structured generative models do not place mass freely across $\Xc$;
they confine samples to a surface carved out by domain knowledge.

\begin{definition}[CMDM]
\label{def:CMDM}
A \emph{Constraint-Manifold Diffusion Model} is specified
by a $C^2$ constraint map $g:\Xc\times\mathcal{C}\to\mathbb{R}^k$ and
a conditioning signal $c\in\mathcal{C}$.
The target distribution $p(\cdot\mid c)$ is supported on the
\emph{constraint manifold}
$\Mc:=\{x\in\Xc:g(x;c)=0\}$.
We assume $\nabla_x g(x;c)$ has full rank $k$ at every $x\in\Mc$
(submersion condition), so by the Implicit Function Theorem $\Mc$ is a
closed $C^2$ submanifold of $\Xc$ of dimension $d_c:=D-k$.
\end{definition}

State-of-the-art CMDM models across various domains are special cases:
Med-DDPM~\citep{dorjsembe2024conditional} ($\Xc=\mathbb{R}^{4\times192\times192\times144}$, $g$
enforces anatomical intensity and continuity), DiffCSP++~\citep{jiao2024space}
($g$ encodes crystallographic space-group symmetry),
Proteina~\citep{geffner2025proteina} ($\Xc=\mathbb{R}^{3N}$, the $C_\alpha$ coordinates of $N$ residues; $g$ encodes
Ramachandran validity and steric clash avoidance), SV4D 2.0~\citep{yao2025sv4d}
and Lyra~\citep{bahmani2026lyra} ($g$ imposes multi-view photometric
consistency).
Despite radically different ambient spaces, all share three properties
our theory exploits: $g$ is externally specified, not learned
\textbf{(P1)}; codimension $k\gg 1$ so a random point in $\Xc$
violates the constraints with probability approaching one \textbf{(P2)};
and $c$ is exact at every diffusion step \textbf{(P3)}.
The constraint map $g$ describes where the data lie: it is a modelling
assumption about the support of $p(\cdot\mid c)$, not an operation
performed at inference. \textsc{ManifoldCache} never evaluates $g$,
projects onto $\Mc$, or checks constraint satisfaction; it reads only
the timestep and the block index. How a model keeps its samples on
$\Mc$ (through training data, conditioning, or an architecture that
builds the constraint in) therefore does not affect how the schedule
is applied. Table~\ref{tab:cmdm} in Appendix~\ref{app:cmdm} lists the
constraint and this mechanism for each of the eight models.

\subsection{Orthogonal Score Decomposition}
\label{sec:general-caching}

\looseness=-1 We adopt a variance-preserving (VP) forward process
$q(x_t\mid x_0)=\mathcal{N}(\sqrt{\abar_t}\,x_0,(1-\abar_t)\mathbf{I}_D)$,
where $\abar_t\in(0,1]$ is strictly decreasing with $\abar_0=1$.
Denoising runs from $t=T$ (near-isotropic noise) to $t=0$ (clean
data on $\Mc$).
The conditional score is
$s_t(x;c):=\nabla_x\log q_t(x\mid c)$,
where $q_t(x\mid c):=\int q(x_t\mid x_0)\,p(x_0\mid c)\,dx_0$.

\textbf{Geometric structure of the score.}
At high noise, $x_t$ almost surely violates $g$; the score must
point the trajectory back toward $\Mc$.
At low noise, $x_t$ is already near $\Mc$, and what remains is to
select which valid sample to produce by moving \emph{within} $\Mc$.
These two roles are geometrically orthogonal.
Since $\Mc$ is a closed $C^2$ submanifold with positive reach
$\rho_c>0$ (Assumption~(A2), \S\ref{app:assump}), Federer's
theorem~\citep{federer1959curvature} guarantees a tubular neighbourhood
of radius $\rho_c$ in which the nearest-point projection
$\Pi:\Xc\to\Mc$ is smooth.
For any $x$ in this tube, the ambient tangent space splits as
$T_x\Xc=T_{\Pi(x)}\Mc\oplus N_{\Pi(x)}\Mc$,
where the normal space
$N_{\Pi(x)}\Mc=\operatorname{span}\{\nabla_x g_j(\Pi(x);c)\}_{j=1}^k$
is spanned by the constraint gradients (well-defined by the submersion
condition).
Letting $P^\perp_x$ and $P^\parallel_x$ denote the orthogonal
projectors onto the normal and tangential subspaces, the conditional
score decomposes as follows~\citep{stanczuk2024diffusion}:
\begin{equation}
  s_t(x;c)
  \;=\;
  \underbrace{P^\perp_x\,s_t(x;c)}_{
    s_t^\perp\;\text{(drives toward }\Mc\text{)}}
  \;+\;
  \underbrace{P^\parallel_x\,s_t(x;c)}_{
    s_t^\parallel\;\text{(navigates within }\Mc\text{)}}.
  \label{eq:score-split}
\end{equation}
As $t\to 0$, the Tweedie posterior mean
$\widehat{x}_0(x_t,t;c):=\mathbb{E}[x_0\mid x_t,c]$ converges onto
$\Mc$ and $g(x_t;c)\to 0$.
We measure the balance between the two components through the
\emph{normal-to-tangential power ratio}
$\kappa_t:=\mathbb{E}\|s_t^\perp\|^2/\mathbb{E}\|s_t^\parallel\|^2\ge 0$.
Theorem~\ref{thm:depth} below shows how $\kappa_t$ determines which
network blocks change least between adjacent steps;
\textsc{ManifoldCache} itself never measures $\kappa_t$ at runtime.

\subsection{The Safe-Caching Threshold \texorpdfstring{$T^*$}{T*}}
\label{sec:optimal-threshold}

Define the log signal-to-noise ratio
$\lambda(t):=\log[\abar_t/(1-\abar_t)]$ and the
\emph{schedule-derived threshold}
\begin{equation}
  T^* \;:=\; \max\{t:\abar_t\ge\tfrac{1}{2}\}
       \;=\; \max\{t:\lambda(t)\ge 0\},
  \label{eq:Tstar}
\end{equation}
\looseness=-1 the latest step at which signal energy is still at least as large as noise energy.
$T^*$ is a scalar computable in closed form from the noise schedule,
independent of $\Mc$, $c$, the network, and any data.

\begin{theorem}[Manifold-aware caching threshold]
\label{thm:threshold}
Under Assumption \textnormal{(A3)} (\S\ref{app:assump}),
let $\Sigma_c$ be the total covariance of $p(\cdot\mid c)$ restricted
to $T\Mc$ with eigenvalues $\{\lambda_i\}_{i=1}^{d_c}$ all satisfying
$\lambda_i\ge 1$ \textnormal{(A3.ii)}, and let
$x_0^\tau:=\widehat{x}_0(x_\tau,\tau;c)$ be the frozen Tweedie
estimate with cache error $\mathcal{E}(\tau):=\mathbb{E}\|x_0^\tau-x_0\|^2$.
Write $\delta_c:=\|\mu_1-\mu_2\|$ for the mode separation,
$\sigma_u^2$ for the within-mode variance along the axis joining the
two modes (\S\ref{app:notation}), $K\in\{1,2\}$ for the mode from which
$x_0$ is drawn, $\Phi$ for the standard Gaussian CDF, and
$p_{\mathrm{flip}}(\tau):=\Pr[\Pi_{\Mc}(x_0^\tau)\ne\mu_K]$ for the
probability that the frozen estimate is closer to the other mode.
For every step $\tau\ge 1$,
\begin{equation}
  p_{\mathrm{flip}}(\tau)
  =\Phi\Bigl(-\frac{\delta_c}{2}\sqrt{\frac{\abar_\tau}{V_\tau^u}}\Bigr),
  \qquad V_\tau^u:=1-\abar_\tau+\abar_\tau\sigma_u^2,
  \label{eq:pflip}
\end{equation}
which is continuous and strictly decreasing in $\abar_\tau$.

\noindent\textbf{Safe} ($\tau\le T^*$): the cache error satisfies
$\mathcal{E}(\tau)\le\sum_{i}(1-\abar_\tau)\lambda_i/
(\abar_\tau\lambda_i+1-\abar_\tau)\le\tfrac{1}{2}\trS$
(Lemma~\ref{lem:posterior-var}), and
$p_{\mathrm{flip}}(\tau)\le\Phi(-\delta_c\sqrt{\abar_\tau}/2)
\le\Phi(-\delta_c/(2\sqrt{2}))$.

\looseness=-1 \noindent\textbf{Unsafe} ($\tau>T^*$): mode confusion has positive
probability,
$p_{\mathrm{flip}}(\tau)>\Phi(-\delta_c/2)=:p_{\mathrm{mix}}>0$,
and this bound does not vanish as $\abar_\tau\uparrow\tfrac{1}{2}$.

\looseness=-1 \noindent\textbf{Role of $T^*$}: $p_{\mathrm{flip}}(\tau)\ge p_{\mathrm{mix}}$
holds if and only if $\abar_\tau\le 1/(2-\sigma_u^2)$. Hence
$\abar_\tau=\tfrac{1}{2}$, the level that defines $T^*$, is the
largest signal level at which $p_{\mathrm{flip}}(\tau)\ge p_{\mathrm{mix}}$
holds for every within-mode spread $\sigma_u\in(0,1]$.
\end{theorem}

\looseness=-1 The safe bound says that freezing the Tweedie estimate at any
$\tau\le T^*$ costs at most half the within-manifold variance, and that
the frozen estimate lands in the wrong mode with probability at most
$\Phi(-\delta_c/(2\sqrt{2}))$, which is small for well-separated modes.
Freezing the estimate is the extreme case of caching, in which nothing
is recomputed after step $\tau$.
The unsafe bound says that at every step $\tau>T^*$, and for every
within-mode spread, more than a fraction $p_{\mathrm{mix}}$ of frozen
estimates land in the wrong mode: when $\abar_\tau<\tfrac{1}{2}$, noise
dominates signal and the noisy state cannot reliably tell the modes
apart.
Because $p_{\mathrm{flip}}$ decreases continuously, $T^*$ does not mark
a sudden jump; it is the largest signal level at which the unsafe bound
is guaranteed for every within-mode spread.
The two-mode model is a simplified stand-in for the multimodal
distributions of real CMDMs (\S\ref{sec:discussion}).
The proof is in \S\ref{sec:proof-t1}; caching is enabled only for
$t\le T^*$.

\subsection{Depth-Adaptive Caching Schedule}
\label{sec:depth-adaptive-schedule}

\looseness=-1 Theorem~\ref{thm:threshold} tells us \emph{when} to cache.
A uniform stride across all blocks wastes the available budget: deep
blocks, operating at coarse resolution and tracking slowly-evolving
global structure, are far more cacheable than shallow blocks tracking
rapidly-changing fine-grained details.
We formalise this with a second theorem.

Let the denoiser consist of $L$ residual blocks $\{f_\ell\}_{\ell=1}^L$
with features $h_\ell(x,t,c)\in\mathbb{R}^{D_\ell}$ and $h_0:=x$.
The \emph{per-block drift}
$\delta_\ell(t):=\|h_\ell(x_{t'},t,c)-h_\ell(x_t,t,c)\|$
measures output change between adjacent steps $t>t'$.
Block $\ell$ is \emph{$\varepsilon$-cacheable with stride $S$} if
$\mathbb{E}\sum_{k=1}^{S-1}\delta_\ell(t-k\Delta t)\le\varepsilon$.

The score decomposition~\eqref{eq:score-split} passes through the
input Jacobian $J_\ell:=\nabla_x h_\ell$: input perturbations
$\Delta x=P^\perp_x\Delta x+P^\parallel_x\Delta x$ produce feature
perturbations $J_\ell\Delta x$.
By Assumption~(A5) (\S\ref{app:assump}), $J_\ell$ has directional upper-Lipschitz constants
$L_\ell^\perp$ and $L_\ell^\parallel$ in the normal and tangential
directions.
Under unit-gain normalisation $L_\ell^\perp+L_\ell^\parallel=1$, set
$\beta_\ell:=L_\ell^\perp\in(0,1)$, the \emph{normal sensitivity} of
block $\ell$.
The depth-ordering condition $\beta_{\ell_1}<\beta_{\ell_2}$ for
$\ell_1<\ell_2$ (Assumption~(A5)) captures the backbone hierarchy; in a
U-Net such as Med-DDPM, shallow blocks at resolution $192{\times}192{\times}144$ must
respond to rapidly changing local structure, while deep blocks at
$12{\times}12{\times}9$ attend to global, constraint-relevant directions.
Define the \emph{drift envelope}
$B(\beta;\kappa):=\beta^2\kappa+(1-\beta)^2$.

\begin{theorem}[Depth-monotone cacheability]
\label{thm:depth}
Under Assumptions \textnormal{(A1)}, \textnormal{(A2)}, and
\textnormal{(A4)--(A6)} (\S\ref{app:assump}), for any $t\le T^*$ the
mean-square drift of every block $\ell$ satisfies
\begin{equation}
  \mathbb{E}\,\delta_{\ell}(t)^2
  \;\le\;
  2B(\beta_{\ell};\kappa_t)\,\Theta_t\,\Delta t^2+O(\Delta t^3),
  \label{eq:drift-ratio-upper}
\end{equation}
where $\Theta_t>0$ is defined in \S\ref{sec:proof-t2}.
For any $\ell_1<\ell_2$ with $\beta_{\ell_1}<\beta_{\ell_2}$, the
envelope of the deeper block is strictly smaller,
$B(\beta_{\ell_2};\kappa_t)<B(\beta_{\ell_1};\kappa_t)$, if and only if
\begin{equation}
  \kappa_t \;<\; \kappa^* \;:=\;
  \frac{(1-\beta_{\ell_1})+(1-\beta_{\ell_2})}
       {\beta_{\ell_1}+\beta_{\ell_2}} \;>\; 0.
  \label{eq:kappa-threshold}
\end{equation}
When \eqref{eq:kappa-threshold} holds, the certified strides
$\widehat{S}_\ell$ of \S\ref{sec:proof-t2} satisfy
\[
  \frac{\widehat{S}_{\ell_2}-1}{\widehat{S}_{\ell_1}-1}
  =\sqrt{\frac{B(\beta_{\ell_1};\kappa_t)}{B(\beta_{\ell_2};\kappa_t)}}\,
  \bigl(1+O(\Delta t)\bigr),
\]
\looseness=-1 which exceeds one for small enough $\Delta t$: deeper blocks receive
strictly longer certified strides.
Under the additional transversality condition \textnormal{(A5$'$)},
the deeper block also has strictly smaller actual drift,
$\mathbb{E}\,\delta_{\ell_2}(t)^2<\mathbb{E}\,\delta_{\ell_1}(t)^2$,
whenever $B(\beta_{\ell_1};\kappa_t)/B(\beta_{\ell_2};\kappa_t)
>2/\bigl((1-\alpha)\gamma_*^2\bigr)$ and $\Delta t$ is small enough,
with $\alpha$ and $\gamma_*$ as in \textnormal{(A5$'$)}.
\end{theorem}

\looseness=-1 The threshold $\kappa^*$ is explicit and depends only on
$\beta_{\ell_1}$ and $\beta_{\ell_2}$.
Theorem~\ref{thm:depth} is conditional: it favours longer strides for
deeper blocks when the tangential part of the score dominates
($\kappa_t<\kappa^*$), and we do not derive this condition from the
noise schedule.
The inversion ablation in \S\ref{sec:results} tests the resulting depth
ordering directly.
The proof is in \S\ref{sec:proof-t2}.

\textbf{The \textsc{ManifoldCache} schedule.}
Both theorems together yield a complete closed-form policy.
Let $\nu_\ell:=(\ell-1)/(L-1)\in[0,1]$ be the normalised block depth
and $1\le S_{\mathrm{sh}}\le S_{\mathrm{dp}}$ be integer strides (e.g.\ 2 and 4).
The stride assigned to block $\ell$ at step $t$ is
\begin{equation}
  S(\ell,t) \;=\; \begin{cases}
    1, & t > T^*, \\[3pt]
    \bigl\lfloor S_{\mathrm{sh}}+(S_{\mathrm{dp}}-S_{\mathrm{sh}})\nu_\ell
    \bigr\rfloor, & t \le T^*,
  \end{cases}
  \label{eq:stride-schedule}
\end{equation}
\looseness=-1 where stride $1$ means always recompute.
Block $\ell$ is reused at step $t\le T^*$ iff
$t_\ell^{\mathrm{last}}-t<S(\ell,t)$, where $t_\ell^{\mathrm{last}}$
is the step at which it was last computed; otherwise it is recomputed
and the cache updated.
The policy costs one integer comparison per block per step:
$T^*$ is read from the noise schedule, $\nu_\ell$ from the block index,
and no calibration or profiling is required.

\section{Implementation and Dataset Details}

\subsection{Models, Datasets, and Metrics}
\label{sec:models-datasets-metrics}

\looseness=-1 We evaluate \textsc{ManifoldCache} on eight CMDMs (\S\ref{sec:constraint-framework}) across five scientific domains using official open-source checkpoints. \textbf{Med-DDPM}~\citep{dorjsembe2024conditional} is a Conv3D U-Net DDPM for segmentation-conditioned semantic 3D brain MRI synthesis on $192{\times}192{\times}144$ BraTS~2021 volumes, evaluated with volumetric 3D \textbf{FID}. \textbf{SV4D~2.0}~\citep{yao2025sv4d} is a multi-view video diffusion transformer with joint 3D/frame attention for monocular-video-to-4D asset generation, evaluated on Objaverse~\citep{yao2025sv4d} over $8$ view pairs using \textbf{PSNR}, \textbf{SSIM}, and \textbf{LPIPS}, computed against full inference (\S\ref{sec:results}). \textbf{Lyra}~\citep{bahmani2026lyra} is an NVIDIA self-distilled video diffusion model featuring a camera-controlled GEN3C-based video DiT for single-image-to-3D Gaussian Splatting scene generation, evaluated using rendered novel views with \textbf{PSNR}, \textbf{SSIM}, and \textbf{LPIPS}, computed against full inference (\S\ref{sec:results}). \textbf{Proteina}~\citep{geffner2025proteina} is a 400M-parameter non-equivariant flow-matching transformer for hierarchical-CATH-conditioned $C_\alpha$ protein backbone generation trained on AFDB, evaluated by the ProteinMPNN$\rightarrow$ESMFold self-consistency pipeline using mean \textbf{scRMSD}, where scRMSD${<}2$\,\AA{} denotes designable samples. \textbf{DiT-MC}~\citep{frank2025sampling} adapts diffusion transformers to graph-conditioned 3D molecular conformer generation and is evaluated on GEOM-QM9 using mean \textbf{COV-R} and \textbf{COV-P} at $0.5$\,\AA{} RMSD. \textbf{UniMoMo}~\citep{kong2025unimomo} is an E(3)-equivariant geometric latent diffusion model for unified peptide, antibody, and small-molecule binder generation from block-graph representations, evaluated on PepBench using mean \textbf{AAR}, median \textbf{C-RMSD}, and median \textbf{L-RMSD}. \textbf{DiffCSP++}~\citep{jiao2024space} is a periodic equivariant diffusion model jointly denoising lattices, fractional coordinates, and atom types under explicit space-group/Wyckoff constraints for crystal structure prediction, evaluated on MP-20 with \textbf{Match Rate} and \textbf{RMSD}. \textbf{DemoDiff}~\citep{liu2026graph} is a $0.7$B-parameter discrete graph diffusion transformer for in-context molecular design on Node-Pair-Encoding motif-tokenised graphs, pretrained on $1.6$M ChEMBL+polymer tasks and evaluated on its 33-task benchmark using task-averaged \textbf{Score} and ICL \textbf{Consistency Score}~\citep{liu2026graph}.

\subsection{Hardware and Hyperparameters}
\label{sec:hyperparameters}
 
All inference experiments are run on a single NVIDIA A100 40GB GPU
with the official inference scripts and default schedulers, samplers,
guidance scales, and step counts of each model.
\textsc{ManifoldCache} is parameterised by $(S_{\mathrm{sh}},S_{\mathrm{dp}})$,
the shallow and deep stride bounds in
Eq.~\eqref{eq:stride-schedule}, and by the cache-enable cutoff which
defaults to $T^*$ from Eq.~\eqref{eq:Tstar}.
We also report
an ablation that enables caching earlier: on Med-DDPM and Proteina,
$(w)$ runs the first fraction $w$ of the denoising steps with full
computation and enables caching for the rest, with
$w\in\{0.1,0.2,0.3,0.4\}$ on Med-DDPM and $w=0.1$ on Proteina.
For UniMoMo and Lyra we use $(2,4)$ and $(2,3)$ respectively;
for all other models we sweep $(S_{\mathrm{sh}},S_{\mathrm{dp}})\in
\{(2,3),(2,4),(2,5),(3,5)\}$, together with the inverted setting
$(4,2)$ as an ablation, and report the configurations shown in the
tables.
 
\subsection{Caching Baselines and Their Applicability}
\label{sec:baselines}
\looseness=-1 The CMDM literature spans four backbone families: volumetric
3D U-Nets (Med-DDPM), $\mathrm{SE}(3)$/E(3)-equivariant networks
(UniMoMo, DiffCSP++), graph and flow-matching transformers (DemoDiff,
DiT-MC, Proteina), and multi-view (4D) video models (SV4D~2.0, Lyra),
and existing caching methods are designed against very
different architectural assumptions.
We therefore restrict each table to baselines that are \emph{actually
runnable} on that backbone with their published assumptions intact;
the rest are categorically inapplicable.
We summarise the per-model selection rationale below; \textbf{\textsc{ManifoldCache} is
backbone-agnostic and runs on all eight.}

\begin{table}[tb]
\centering
\caption{\textbf{Proteina}~\citep{geffner2025proteina} (de novo
$C_\alpha$ backbone generation, flow-matching DiT trained on
AFDB): runtime, mean self-consistency RMSD (\textbf{scRMSD};
ProteinMPNN $\to$ ESMFold), and peak \textbf{VRAM}.
The $(w)$ row runs the first fraction $w$ of the 100 steps with full
computation before enabling caching; $(T^*)$ corresponds to $w=0.15$.
\textsc{ManifoldCache} $(T^*)\,(2,4)$ has the lowest scRMSD among the
accelerators (2.04 vs.\ 1.93 for full inference).}
\label{tab:runtime_scrmsd}
\small
\setlength{\tabcolsep}{8pt}
\renewcommand{\arraystretch}{0.7} 
\scalebox{0.85}{
\begin{tabular}{l c c c}
\toprule
\textbf{Method} & \textbf{Time (sec)} $\downarrow$ & \textbf{Mean scRMSD} $\downarrow$ & \textbf{VRAM (MiB)} $\downarrow$ \\
\midrule
Full Inference                                   & 90.47 & \textbf{1.93} & 15438 \\
\midrule
SmoothCache~\citep{liu2025smoothcache}           & 60.96 & 2.11 & 16430 \\
FORA~\citep{selvaraju2024fora}                   & 70.62 & 2.31 & 15980 \\
TeaCache~\citep{liu2025timestep}                 & 61.10 & 8.07 & 17800 \\
TaylorSeer~\citep{liu2025reusing}                & 61.08 & 9.11 & 16700 \\
\midrule
\rowcolor{gray!15}
\textsc{ManifoldCache} $(w{=}0.1)\,(2,4)$        & 45.12 & 2.35 & 16000 \\
\midrule
\rowcolor{gray!15}
\textsc{ManifoldCache} $(T^*)\,(3,5)$            & 43.76 & 2.85 & 16190 \\
\rowcolor{gray!15}
\textsc{ManifoldCache} $(T^*)\,(2,5)$            & 45.10 & 2.61 & 16197 \\
\rowcolor{gray!15}
\textsc{ManifoldCache} $(T^*)\,(2,4)$            & 46.78 & \textbf{2.04} & 16190 \\
\bottomrule
\end{tabular}}
\end{table}

\looseness=-1 For U-Net-style CMDMs, including \textbf{Med-DDPM} and the convolutional VAE branches of \textbf{SV4D~2.0} and \textbf{Lyra}, we use \textbf{DeepCache}~\citep{ma2024deepcache} and, for Med-DDPM, \textbf{FasterDiffusion}~\citep{li2024faster}, which explicitly cache encoder/decoder or skip-connected U-Net features. Transformer-based CMDMs such as \textbf{Proteina}, \textbf{DiT-MC}, and \textbf{DemoDiff} instead use DiT-native cachers including \textbf{FORA}~\citep{selvaraju2024fora}, \textbf{TeaCache}~\citep{liu2025timestep}, \textbf{SmoothCache}~\citep{liu2025smoothcache}, \textbf{TaylorSeer}~\citep{liu2025reusing}, \textbf{ToCa}~\citep{zou2024accelerating}, \textbf{AB-Cache}~\citep{yu2025ab}, and \textbf{PFC}~\citep{sommer2025predictive}, since these operate on isotropic transformer block outputs or token-feature trajectories compatible with graph and flow-matching transformers. \textbf{DemoDiff} additionally uses \textbf{dKV-Cache}~\citep{ma2025dkvcache}, the only published cacher designed for discrete masked-token diffusion with bidirectional attention. For \textbf{UniMoMo} we use \textbf{SmoothCache} and \textbf{DeepCache}; other DiT cachers violate its E(3)-equivariant message-passing structure. For \textbf{DiffCSP++} we compare against \textbf{DeepCache} only, because current DiT/video-DiT methods assume non-periodic isotropic attention or separable spatial--temporal attention absent in periodic equivariant diffusion. Methods such as \textbf{PAB}~\citep{zhao2408real} and \textbf{AdaCache}~\citep{kahatapitiya2025adaptive} are only applicable to video DiTs with explicit spatial/temporal attention splits, while U-Net cachers are excluded from isotropic transformers lacking skip connections, and continuous-trajectory methods are unsuitable for discrete motif-token diffusion. We use PAB and AdaCache for Lyra, and we additionally compare with \textbf{FrDiff}~\citep{so2024frdiff} on SV4D~2.0.

\begin{table}[tb]
\centering
\caption{\textbf{DiT-MC}~\citep{frank2025sampling} (3D molecular
conformer generation on GEOM-QM9, graph-conditioned DiT):
runtime, mean recall (\textbf{COV-R}) and precision
(\textbf{COV-P}) at a $0.5$\,\AA{} threshold, and
\textbf{VRAM}.}
\label{tab:cov_metrics_clean}
\small
\setlength{\tabcolsep}{8pt}
\renewcommand{\arraystretch}{0.85}
\scalebox{0.805}{
\begin{tabular}{l c c c c}
\toprule
\textbf{Method} & \textbf{Time (sec)} $\downarrow$ & \textbf{COV-R (mean)} $\uparrow$ & \textbf{COV-P (mean)} $\uparrow$ & \textbf{VRAM (MiB)} $\downarrow$ \\
\midrule
Full Inference                        & 621 & \textbf{0.9695} & \textbf{0.9569} & 30706 \\
TeaCache~\citep{liu2025timestep}                 & 565 & 0.9473 &  0.9424 & 32491 \\
AB-Cache~\citep{yu2025ab}                        & 501 & 0.9460 &  0.9430 & 31900 \\
PFC (ABC-Forecasting)~\citep{sommer2025predictive} & 561 & 0.9517 & 0.9428 & 32998 \\
\midrule
  \rowcolor{gray!15}
\textsc{ManifoldCache} $(T^*)\,(2,4)$            & \textbf{377.3} & 0.9693 & 0.9563  & 32360 \\
\bottomrule
\end{tabular}
}
\end{table}

\section{Results and Analysis}
\label{sec:results}

We report runtime and generation quality through various metrics across all eight CMDMs in
Tables~\ref{tab:runtime_scrmsd}--\ref{tab:viewpair_metrics}.
Each table is ordered as follows: Full Inference, applicable caching baselines, and then \textsc{ManifoldCache} configurations.
Figure~\ref{fig:tradeoff-all} (Appendix~\ref{app:tradeoff}) plots the quality--speed trade-off for all eight models, Appendix~\ref{app:flops} reports FLOPs on six models and NFE on two, and Appendix~\ref{app:spec} compares \textsc{ManifoldCache} with speculative sampling.
\textbf{\textsc{ManifoldCache} consistently delivers the best quality--speed
trade-off, often \emph{strictly} dominating every baseline on both
axes simultaneously.}

\looseness=-1 On Med-DDPM (Table~\ref{tab:medddpm_final}), enabling caching earlier
than $T^*$ worsens FID at every point of the warm-up sweep: FID is
$4.55$ at $T^*$ and rises to $4.83$, $5.51$, $7.01$, and $12.56$ as
the warm-up fraction $w$ shrinks from $0.4$ to $0.1$.
On Proteina (Table~\ref{tab:runtime_scrmsd}), enabling caching
before $T^*$ ($w{=}0.1$, whereas $T^*$ corresponds to $w{=}0.15$) raises
mean scRMSD from $2.04$ to $2.35$.
This is the behaviour Theorem~\ref{thm:threshold} predicts:
\textbf{reusing features while noise still dominates ($t>T^*$) risks
mode confusion}.
When more speedup is needed and some loss of quality is acceptable, the
warm-up can still be tuned as a deliberate trade-off: on Med-DDPM
(Table~\ref{tab:medddpm_final}), shortening it from $T^*$ to $w{=}0.1$
raises the speedup from $1.4\times$ to $1.9\times$, with FID rising from
$4.55$ to $12.56$.

\begin{table}[tb]
\centering
\caption{\textbf{UniMoMo}~\citep{kong2025unimomo} (peptide binder
design via E(3)-equivariant latent diffusion): runtime,
amino-acid recovery (\textbf{AAR}, mean), complex (\textbf{C-RMSD},
median, $C_\alpha$ after target alignment), ligand
(\textbf{L-RMSD}, median, after peptide alignment), and \textbf{VRAM}.}
\label{tab:unimomo_metrics}
\small
\setlength{\tabcolsep}{8pt}
\renewcommand{\arraystretch}{0.85}
\scalebox{0.88}{
\begin{tabular}{l c c c c c}
\toprule
\textbf{Method} & \textbf{Time (sec)} $\downarrow$ & \textbf{AAR} $\uparrow$ & \textbf{C-RMSD} $\downarrow$ & \textbf{L-RMSD} $\downarrow$ & \textbf{VRAM (MiB)} $\downarrow$ \\
\midrule
Full Inference                                   & 687 & \textbf{0.2562} & \textbf{3.14} & \textbf{1.59} & 28072 \\
\midrule
SmoothCache~\citep{liu2025smoothcache}           & 578 & 0.1890 & 6.10 & 6.04 & 28126 \\
DeepCache~\citep{ma2024deepcache}                & 555 & 0.1622 & 4.98 & 3.10 & 28960 \\
\midrule
  \rowcolor{gray!15}
\textsc{ManifoldCache} $(T^*)\,(2,4)$            & \textbf{540} & 0.2550 & 3.29 & 1.75 & 28148 \\
\bottomrule
\end{tabular}
}
\end{table}

\begin{table}[tb]
\centering
\caption{\textbf{DiffCSP++}~\citep{jiao2024space} (space-group
constrained crystal structure prediction on MP-20):
\textbf{Match Rate} and \textbf{RMSD}, runtime, and \textbf{VRAM}.
\textsc{ManifoldCache} $(T^*)\,(2,4)$ has the highest Match Rate and the
lowest RMSD among the accelerators.}
\label{tab:match_rmsd}
\small
\setlength{\tabcolsep}{8pt}
\renewcommand{\arraystretch}{.85}
\scalebox{0.85}{
\begin{tabular}{l c c c c}
\toprule
\textbf{Method} & \textbf{Match Rate} $\uparrow$ & \textbf{RMSD} $\downarrow$ & \textbf{Time (sec)} $\downarrow$ & \textbf{VRAM (MiB)} $\downarrow$ \\
\midrule
Full Inference                                   & \textbf{0.7840} & \textbf{0.0289} & 1269 & 1886 \\
\midrule
DeepCache~\citep{ma2024deepcache}                & 0.7370 & 0.0464 & 957 & 1894 \\
\midrule
  \rowcolor{gray!15}
\textsc{ManifoldCache} $(T^*)\,(3,5)$            & 0.7450 & 0.0459 & \textbf{883} & 2014 \\
  \rowcolor{gray!15}
\textsc{ManifoldCache} $(T^*)\,(2,4)$            & \textbf{0.7550} & \textbf{0.0327} & 938 & 2014 \\
\bottomrule
\end{tabular}
}
\end{table}


\looseness=-1 The $(S_{\mathrm{sh}},S_{\mathrm{dp}})=(4,2)$ \emph{inversion}
ablation
isolates whether deeper-block-skipping is the source of the gain
predicted by Theorem~\ref{thm:depth}.
On Med-DDPM (Table~\ref{tab:medddpm_final}), $(4,2)$ raises FID from
$4.55$ to $5.63$. Because Med-DDPM's shallow blocks run at full
resolution, the inverted schedule also skips more computation
($107$\,s vs $121$\,s per sample), so we compare at matched cost:
$(T^*)\,(3,5)$ reaches FID $5.29$ at $109$\,s and $(w{=}0.3)\,(2,4)$
reaches $5.51$ at $101$\,s, both better than $(4,2)$.
\textbf{Giving shallow blocks the longer stride is more
costly than giving it to deep blocks}, which matches the depth ordering
of Theorem~\ref{thm:depth}. All baseline schedules are taken as the optimal settings reported in their respective papers. We mainly report \textsc{ManifoldCache} under the schedules (2,3), (2,4), and (2,5), as these consistently yielded quality generation comparable to the baselines while spanning the most informative range of speedup–fidelity trade-offs.

 \begin{table}[tb]
\centering
\caption{\textbf{Med-DDPM}~\citep{dorjsembe2024conditional} (semantic
3D brain MRI synthesis on BraTS, $192{\times}192{\times}144$ Conv3D U-Net):
runtime per sample and 3D \textbf{FID} between generated and real
volumes.
$(w)$ rows run the first fraction $w$ of the 250 steps with full
computation before enabling caching;
\textsc{ManifoldCache} $(T^*)\,(2,4)$ is $1.4{\times}$ faster than
full inference \emph{and} attains the lowest FID among all
accelerators; Figure~\ref{fig:med} shows samples.}
\label{tab:medddpm_final}
\small
\setlength{\tabcolsep}{8pt}
\renewcommand{\arraystretch}{0.85}
\scalebox{0.85}{
\begin{tabular}{l c c c}
\toprule
\textbf{Method} & \textbf{Time per sample (sec)} $\downarrow$ & \textbf{FID} $\downarrow$ & \textbf{VRAM (MiB)} $\downarrow$ \\
\midrule
Full Inference                                   & 165 & \textbf{1.21} & 22392 \\
\midrule
DeepCache~\citep{ma2024deepcache}                & 143 & 20.14 & 22994 \\
FasterDiffusion~\citep{li2024faster}             & 123 & 13.37 & 25780 \\
\midrule
  \rowcolor{gray!15}
\textsc{ManifoldCache} $(w{=}0.1)\,(2,4)$        & \phantom{0}85 & 12.56 & 30411 \\
  \rowcolor{gray!15}
\textsc{ManifoldCache} $(w{=}0.2)\,(2,4)$        & \phantom{0}95 & 7.01 & 30498 \\
  \rowcolor{gray!15}
\textsc{ManifoldCache} $(w{=}0.3)\,(2,4)$        & 101 & 5.51 & 30498 \\
  \rowcolor{gray!15}
\textsc{ManifoldCache} $(w{=}0.4)\,(2,4)$        & 110 & 4.83 & 30497 \\
\midrule
  \rowcolor{gray!15}
\textsc{ManifoldCache} $(T^*)\,(4,2)$            & 107 & 5.63 & 30531 \\
  \rowcolor{gray!15}
\textsc{ManifoldCache} $(T^*)\,(3,5)$            & 109 & 5.29 & 30528 \\
  \rowcolor{gray!15}
\textsc{ManifoldCache} $(T^*)\,(2,5)$            & 120 & 4.57 & 30550 \\
  \rowcolor{gray!15}
\textsc{ManifoldCache} $(T^*)\,(2,4)$            & 121 & \textbf{4.55} & 30611 \\
\bottomrule
\end{tabular}}
\end{table}

\begin{table}[tb]
\centering
\caption{\textbf{DemoDiff}~\citep{liu2026graph} (in-context molecular
design, 0.7B discrete graph diffusion transformer with motif-token
NPE): runtime, official task-averaged design \textbf{Score}, and ICL
\textbf{Consistency Score}.
\textsc{ManifoldCache} $(T^*)\,(2,4)$ retains $88\%$ of the design
Score at $1.5{\times}$ speedup; alternatives lose $24{-}46\%$.}
\label{tab:caching_metrics}
\small
\setlength{\tabcolsep}{8pt}
\renewcommand{\arraystretch}{0.85}
\scalebox{0.9}{
\begin{tabular}{l c c c c}
\toprule
\textbf{Method} & \textbf{Time (sec)} $\downarrow$ & \textbf{Score} $\uparrow$ & \textbf{Consistency Score} $\uparrow$  & \textbf{VRAM (MiB)} $\downarrow$ \\
\midrule
Full Inference                                   & 1323.49 & \textbf{0.8050} & \textbf{0.2083} & 15584 \\
\midrule
FORA~\citep{selvaraju2024fora}                   & \phantom{0}867.61 & 0.5970 & 0.0799 & 17761 \\
ToCa~\citep{zou2024accelerating}                 & \phantom{0}975.87 & 0.6118 & 0.1237 &  18450\\
dKV-Cache~\citep{ma2025dkvcache}                 & 1111.50 & 0.4381 & 0.0780 & 16109 \\
\midrule
  \rowcolor{gray!15}
\textsc{ManifoldCache} $(T^*)\,(2,4)$            & \phantom{0}866.00 & 0.7048 & 0.1443 & 18119\\
\bottomrule
\end{tabular}
}
\end{table}

\looseness=-1 On Proteina (Table~\ref{tab:runtime_scrmsd}), TeaCache and TaylorSeer, both relying on continuous
feature-trajectory extrapolation, reach scRMSD $8.07$ and $9.11$
respectively, $4{-}5\times$ worse than full inference.
This is not a mild quality loss but a regime change. A likely cause is
that forecasting extrapolates features without any correction toward
$\Mc$, which accelerates drift off the manifold.
In contrast, \textsc{ManifoldCache} at $(2,4)$ delivers scRMSD
$2.04$, close to the $1.93$ of full inference, while still
running $1.9\times$ faster.
The same pattern holds for DemoDiff (Table~\ref{tab:caching_metrics}), where dKV-Cache (the
only competing discrete-diffusion cacher) drops the design Score by
$\sim\!46\%$ relative to full inference.
 
\looseness=-1 On Lyra (Table~\ref{tab:video_metrics} and Figure~\ref{fig:lyra}), \textsc{ManifoldCache} achieves PSNR $39.27$ at $3416$\,s
versus DeepCache's PSNR $31.42$ at $3522$\,s and AdaCache's PSNR
$34.35$ at $3491$\,s, making it simultaneously \emph{faster and
$\sim\!5{-}8$ dB higher in PSNR}.
On SV4D~2.0 (Table~\ref{tab:viewpair_metrics}), \textsc{ManifoldCache} reaches PSNR $40.75$ /
LPIPS $0.058$ in $336$\,s, an absolute improvement of $13.9$ dB
PSNR and $0.131$ LPIPS over DeepCache at slightly lower latency.
\textit{Such joint gains in speed and quality over caching baselines
suggest that aligning the cache schedule with the structure of the
score avoids much of the usual speed--quality trade-off.} Note that for SV4D~2.0 and Lyra, we treat the full-inference run as ground truth, so it reports only wall-clock time and no quality metrics, while all baseline and \textsc{ManifoldCache} numbers are computed against this full-inference reference.
Several baselines are both slower and less accurate than
\textsc{ManifoldCache}: they are not designed to restrict reuse to
$t\le T^*$, and our warm-up ablations show that reuse during the
high-noise steps is costly; reducing their reuse to recover quality
would make them slower still.

Caching incurs an unavoidable activation-storage cost proportional to
cached feature volume~\citep{lv2024fastercache}. Across
the tables, \textsc{ManifoldCache} uses more VRAM than full inference,
from $0.3\%$ more on UniMoMo to $37\%$ more on Med-DDPM, where it
also uses more than every baseline; on the other models
its memory use is comparable to that of the caching baselines.

\begin{table}[tb]
\centering
\caption{\textbf{Lyra}~\citep{bahmani2026lyra} (single-image $\to$
3D Gaussian Splatting via camera-controlled video DiT, evaluated on multi-view synthesis of dynamic scenes): \textbf{PSNR}, \textbf{SSIM},
\textbf{LPIPS}, and runtime.
\textsc{ManifoldCache} has the best PSNR, SSIM, and runtime; PAB has a
slightly lower LPIPS but is $1.4{\times}$ slower.}
\label{tab:video_metrics}
\small
\setlength{\tabcolsep}{8pt}
\renewcommand{\arraystretch}{0.85}
\scalebox{0.85}{
\begin{tabular}{l c c c c}
\toprule
\textbf{Method} & \textbf{PSNR} $\uparrow$ & \textbf{SSIM} $\uparrow$ & \textbf{LPIPS} $\downarrow$ & \textbf{Time (sec)} $\downarrow$ \\
\midrule
Full Inference                                   & --    & --    & --    & 5577.64 \\
\midrule
DeepCache~\citep{ma2024deepcache}                & 31.42 & 0.9416 & 0.0890 & 3522.21 \\
PAB~\citep{zhao2408real}                         & 36.85 & 0.9750 & \textbf{0.0552} & 4744.55 \\
AdaCache~\citep{kahatapitiya2025adaptive}        & 34.35 & 0.9734 & 0.0910 & 3491.62 \\
\midrule
  \rowcolor{gray!15}
\textsc{ManifoldCache} $(T^*)\,(2,3)$            & \textbf{39.27} & \textbf{0.9817} & 0.0564 & \textbf{3416.05} \\
\bottomrule
\end{tabular}}
\end{table}

\begin{table}[tb]
\centering
\caption{\textbf{SV4D~2.0}~\citep{yao2025sv4d} (multi-view video
diffusion for dynamic 4D asset generation, 8 view pairs on
Objaverse): \textbf{PSNR}, \textbf{SSIM}, \textbf{LPIPS}, and
runtime.
\textsc{ManifoldCache} $(T^*)\,(2,3)$ is fastest \emph{and} delivers
$+13.9$\,dB PSNR over DeepCache.}
\label{tab:viewpair_metrics}
\small
\setlength{\tabcolsep}{8pt}
\renewcommand{\arraystretch}{0.85}
\scalebox{0.8}{
\begin{tabular}{l c c c c c c}
\toprule
\textbf{Method} & \textbf{PSNR} $\uparrow$ & \textbf{SSIM} $\uparrow$ & \textbf{LPIPS} $\downarrow$ & \textbf{Time (sec)} $\downarrow$ & \textbf{VRAM (MiB)} $\downarrow$ \\
\midrule
Full Inference                                   & --    & --    & --    & 409 & 16520 \\
\midrule
DeepCache~\citep{ma2024deepcache}                & 26.88 & 0.8730 & 0.1891 & 350 & 17885\\
FrDiff~\citep{so2024frdiff}    & 31.94 & 0.9192 & 0.0764 & 347 & 18673 \\
\midrule
  \rowcolor{gray!15}
\textsc{ManifoldCache} $(T^*)\,(2,3)$            & \textbf{40.75} & \textbf{0.9909} & \textbf{0.0579} & \textbf{336} & 18540 \\
\bottomrule
\end{tabular}
}
\end{table}

\textbf{The configuration sweet spot is consistent.}
$(2,4)$ at cutoff $T^*$ is our default and the best of the reported
configurations on Proteina, DiffCSP++, and Med-DDPM.
Other strides expose interpretable trade-offs: $(2,5)$ pushes deep
strides further and gains marginal speed at marginal quality
cost (Med-DDPM FID $4.57$ vs $4.55$ at $(2,4)$); $(2,3)$ is the
most conservative setting and is the
setting we use for SV4D~2.0 and Lyra.
This stability across radically different ambient spaces, from $192{\times}192{\times}144$
voxel grids to motif graphs, crystal lattices, and multi-view (4D)
video latents, empirically supports the universality claim of
\S\ref{sec:constraint-framework}: a single closed-form schedule, derived
purely from the noise schedule and block depth, governs cacheability
across CMDMs. Please refer to Appendix~\ref{limitt} for detailed
Limitations and Future Work.

\begin{figure*}[h]
    \centering
    \includegraphics[width=0.6\textwidth]{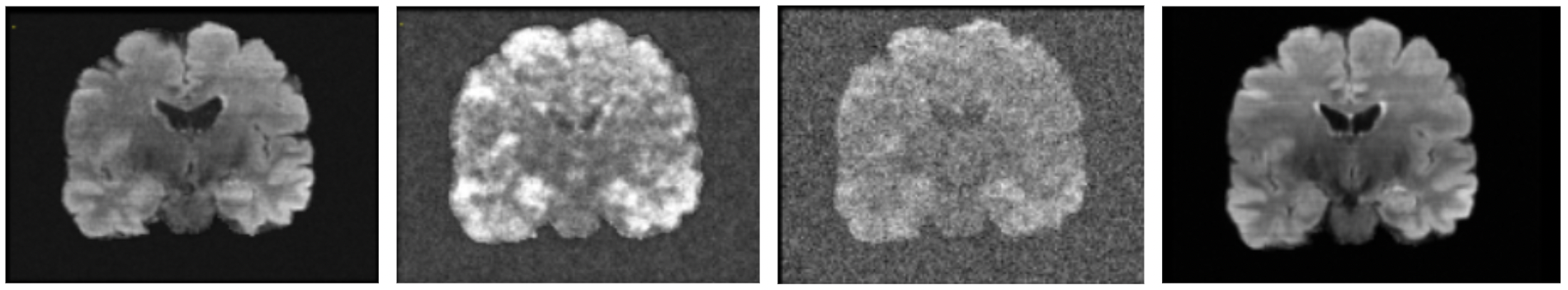}
    \caption{\small From left to right: \textsc{ManifoldCache}, FasterDiffusion, DeepCache, and Full Inference. \textsc{ManifoldCache} achieves visual quality most similar to Full Inference on Med-DDPM (3D Brain MRI).}
    \label{fig:med}
\end{figure*}

\section{Conclusion}
We introduced \textsc{ManifoldCache}, the first training-free acceleration framework designed specifically for Constraint-Manifold Diffusion Models (CMDMs). By exploiting the orthogonal decomposition of the diffusion score into manifold-normal and manifold-tangential components, \textsc{ManifoldCache} derives a principled safe-caching threshold directly from the noise schedule together with a depth-adaptive caching policy grounded in score geometry. The resulting method requires no retraining, no calibration data, and negligible runtime overhead. Across eight diverse CMDMs spanning medical imaging, molecular generation, protein design, crystal prediction, and multi-view 4D synthesis, \textsc{ManifoldCache} consistently achieves superior quality--speed trade-offs and often simultaneously improves both inference time and generation fidelity over strong existing accelerators. Our results suggest that constraint geometry is not only central to generation quality, but also a fundamental signal for determining when and where diffusion computation can be safely reused.


\clearpage 
\bibliographystyle{plainnat}
\bibliography{neurips}

\newpage
\restorefloatsep
\section*{NeurIPS Paper Checklist}

\begin{enumerate}

\item {\bf Claims}
    \item[] Question: Do the main claims made in the abstract and introduction accurately reflect the paper's contributions and scope?
    \item[] Answer: \answerYes{} 
    \item[] Justification: This can be verified by the abstract and introduction.
    \item[] Guidelines:
    \begin{itemize}
        \item The answer \answerNA{} means that the abstract and introduction do not include the claims made in the paper.
        \item The abstract and/or introduction should clearly state the claims made, including the contributions made in the paper and important assumptions and limitations. A \answerNo{} or \answerNA{} answer to this question will not be perceived well by the reviewers. 
        \item The claims made should match theoretical and experimental results, and reflect how much the results can be expected to generalize to other settings. 
        \item It is fine to include aspirational goals as motivation as long as it is clear that these goals are not attained by the paper. 
    \end{itemize}

\item {\bf Limitations}
    \item[] Question: Does the paper discuss the limitations of the work performed by the authors?
    \item[] Answer: \answerYes{}{} 
    \item[] Justification: {In the Appendix}
    \item[] Guidelines:
    \begin{itemize}
        \item The answer \answerNA{} means that the paper has no limitation while the answer \answerNo{} means that the paper has limitations, but those are not discussed in the paper. 
        \item The authors are encouraged to create a separate ``Limitations'' section in their paper.
        \item The paper should point out any strong assumptions and how robust the results are to violations of these assumptions (e.g., independence assumptions, noiseless settings, model well-specification, asymptotic approximations only holding locally). The authors should reflect on how these assumptions might be violated in practice and what the implications would be.
        \item The authors should reflect on the scope of the claims made, e.g., if the approach was only tested on a few datasets or with a few runs. In general, empirical results often depend on implicit assumptions, which should be articulated.
        \item The authors should reflect on the factors that influence the performance of the approach. For example, a facial recognition algorithm may perform poorly when image resolution is low or images are taken in low lighting. Or a speech-to-text system might not be used reliably to provide closed captions for online lectures because it fails to handle technical jargon.
        \item The authors should discuss the computational efficiency of the proposed algorithms and how they scale with dataset size.
        \item If applicable, the authors should discuss possible limitations of their approach to address problems of privacy and fairness.
        \item While the authors might fear that complete honesty about limitations might be used by reviewers as grounds for rejection, a worse outcome might be that reviewers discover limitations that aren't acknowledged in the paper. The authors should use their best judgment and recognize that individual actions in favor of transparency play an important role in developing norms that preserve the integrity of the community. Reviewers will be specifically instructed to not penalize honesty concerning limitations.
    \end{itemize}

\item {\bf Theory assumptions and proofs}
    \item[] Question: For each theoretical result, does the paper provide the full set of assumptions and a complete (and correct) proof?
    \item[] Answer: \answerYes{} 
    \item[] Justification: {In the appendix}
    \item[] Guidelines:
    \begin{itemize}
        \item The answer \answerNA{} means that the paper does not include theoretical results. 
        \item All the theorems, formulas, and proofs in the paper should be numbered and cross-referenced.
        \item All assumptions should be clearly stated or referenced in the statement of any theorems.
        \item The proofs can either appear in the main paper or the supplemental material, but if they appear in the supplemental material, the authors are encouraged to provide a short proof sketch to provide intuition. 
        \item Inversely, any informal proof provided in the core of the paper should be complemented by formal proofs provided in appendix or supplemental material.
        \item Theorems and Lemmas that the proof relies upon should be properly referenced. 
    \end{itemize}

    \item {\bf Experimental result reproducibility}
    \item[] Question: Does the paper fully disclose all the information needed to reproduce the main experimental results of the paper to the extent that it affects the main claims and/or conclusions of the paper (regardless of whether the code and data are provided or not)?
    \item[] Answer: \answerYes{} 
    \item[] Justification: {Implementation details added}
    \item[] Guidelines:
    \begin{itemize}
        \item The answer \answerNA{} means that the paper does not include experiments.
        \item If the paper includes experiments, a \answerNo{} answer to this question will not be perceived well by the reviewers: Making the paper reproducible is important, regardless of whether the code and data are provided or not.
        \item If the contribution is a dataset and\slash or model, the authors should describe the steps taken to make their results reproducible or verifiable. 
        \item Depending on the contribution, reproducibility can be accomplished in various ways. For example, if the contribution is a novel architecture, describing the architecture fully might suffice, or if the contribution is a specific model and empirical evaluation, it may be necessary to either make it possible for others to replicate the model with the same dataset, or provide access to the model. In general. releasing code and data is often one good way to accomplish this, but reproducibility can also be provided via detailed instructions for how to replicate the results, access to a hosted model (e.g., in the case of a large language model), releasing of a model checkpoint, or other means that are appropriate to the research performed.
        \item While NeurIPS does not require releasing code, the conference does require all submissions to provide some reasonable avenue for reproducibility, which may depend on the nature of the contribution. For example
        \begin{enumerate}
            \item If the contribution is primarily a new algorithm, the paper should make it clear how to reproduce that algorithm.
            \item If the contribution is primarily a new model architecture, the paper should describe the architecture clearly and fully.
            \item If the contribution is a new model (e.g., a large language model), then there should either be a way to access this model for reproducing the results or a way to reproduce the model (e.g., with an open-source dataset or instructions for how to construct the dataset).
            \item We recognize that reproducibility may be tricky in some cases, in which case authors are welcome to describe the particular way they provide for reproducibility. In the case of closed-source models, it may be that access to the model is limited in some way (e.g., to registered users), but it should be possible for other researchers to have some path to reproducing or verifying the results.
        \end{enumerate}
    \end{itemize}

\item {\bf Open access to data and code}
    \item[] Question: Does the paper provide open access to the data and code, with sufficient instructions to faithfully reproduce the main experimental results, as described in supplemental material?
    \item[] Answer: \answerYes{} 
    \item[] Justification: {Datasets are public, settings and splits are provided, code will be released upon acceptance}
    \item[] Guidelines:
    \begin{itemize}
        \item The answer \answerNA{} means that paper does not include experiments requiring code.
        \item Please see the NeurIPS code and data submission guidelines (\url{https://neurips.cc/public/guides/CodeSubmissionPolicy}) for more details.
        \item While we encourage the release of code and data, we understand that this might not be possible, so \answerNo{} is an acceptable answer. Papers cannot be rejected simply for not including code, unless this is central to the contribution (e.g., for a new open-source benchmark).
        \item The instructions should contain the exact command and environment needed to run to reproduce the results. See the NeurIPS code and data submission guidelines (\url{https://neurips.cc/public/guides/CodeSubmissionPolicy}) for more details.
        \item The authors should provide instructions on data access and preparation, including how to access the raw data, preprocessed data, intermediate data, and generated data, etc.
        \item The authors should provide scripts to reproduce all experimental results for the new proposed method and baselines. If only a subset of experiments are reproducible, they should state which ones are omitted from the script and why.
        \item At submission time, to preserve anonymity, the authors should release anonymized versions (if applicable).
        \item Providing as much information as possible in supplemental material (appended to the paper) is recommended, but including URLs to data and code is permitted.
    \end{itemize}

\item {\bf Experimental setting/details}
    \item[] Question: Does the paper specify all the training and test details (e.g., data splits, hyperparameters, how they were chosen, type of optimizer) necessary to understand the results?
    \item[] Answer: \answerYes{} 
    \item[] Justification: {Details provided in the main paper}
    \item[] Guidelines:
    \begin{itemize}
        \item The answer \answerNA{} means that the paper does not include experiments.
        \item The experimental setting should be presented in the core of the paper to a level of detail that is necessary to appreciate the results and make sense of them.
        \item The full details can be provided either with the code, in appendix, or as supplemental material.
    \end{itemize}

\item {\bf Experiment statistical significance}
    \item[] Question: Does the paper report error bars suitably and correctly defined or other appropriate information about the statistical significance of the experiments?
    \item[] Answer: \answerNA{} 
    \item[] Justification: {All existing methods in this category report numbers without error bars, so we followed them.}
    \item[] Guidelines:
    \begin{itemize}
        \item The answer \answerNA{} means that the paper does not include experiments.
        \item The authors should answer \answerYes{} if the results are accompanied by error bars, confidence intervals, or statistical significance tests, at least for the experiments that support the main claims of the paper.
        \item The factors of variability that the error bars are capturing should be clearly stated (for example, train/test split, initialization, random drawing of some parameter, or overall run with given experimental conditions).
        \item The method for calculating the error bars should be explained (closed form formula, call to a library function, bootstrap, etc.)
        \item The assumptions made should be given (e.g., Normally distributed errors).
        \item It should be clear whether the error bar is the standard deviation or the standard error of the mean.
        \item It is OK to report 1-sigma error bars, but one should state it. The authors should preferably report a 2-sigma error bar than state that they have a 96\% CI, if the hypothesis of Normality of errors is not verified.
        \item For asymmetric distributions, the authors should be careful not to show in tables or figures symmetric error bars that would yield results that are out of range (e.g., negative error rates).
        \item If error bars are reported in tables or plots, the authors should explain in the text how they were calculated and reference the corresponding figures or tables in the text.
    \end{itemize}

\item {\bf Experiments compute resources}
    \item[] Question: For each experiment, does the paper provide sufficient information on the computer resources (type of compute workers, memory, time of execution) needed to reproduce the experiments?
    \item[] Answer: \answerYes{} 
    \item[] Justification: {Provided in main paper}
    \item[] Guidelines:
    \begin{itemize}
        \item The answer \answerNA{} means that the paper does not include experiments.
        \item The paper should indicate the type of compute workers CPU or GPU, internal cluster, or cloud provider, including relevant memory and storage.
        \item The paper should provide the amount of compute required for each of the individual experimental runs as well as estimate the total compute. 
        \item The paper should disclose whether the full research project required more compute than the experiments reported in the paper (e.g., preliminary or failed experiments that didn't make it into the paper). 
    \end{itemize}
    
\item {\bf Code of ethics}
    \item[] Question: Does the research conducted in the paper conform, in every respect, with the NeurIPS Code of Ethics \url{https://neurips.cc/public/EthicsGuidelines}?
    \item[] Answer: \answerYes{} 
    \item[] Justification: The paper conforms to the Code of Ethics.
    \item[] Guidelines:
    \begin{itemize}
        \item The answer \answerNA{} means that the authors have not reviewed the NeurIPS Code of Ethics.
        \item If the authors answer \answerNo, they should explain the special circumstances that require a deviation from the Code of Ethics.
        \item The authors should make sure to preserve anonymity (e.g., if there is a special consideration due to laws or regulations in their jurisdiction).
    \end{itemize}

\item {\bf Broader impacts}
    \item[] Question: Does the paper discuss both potential positive societal impacts and negative societal impacts of the work performed?
    \item[] Answer: \answerNA{} 
    \item[] Justification: There is no societal impact of the work performed.
    \item[] Guidelines:
    \begin{itemize}
        \item The answer \answerNA{} means that there is no societal impact of the work performed.
        \item If the authors answer \answerNA{} or \answerNo, they should explain why their work has no societal impact or why the paper does not address societal impact.
        \item Examples of negative societal impacts include potential malicious or unintended uses (e.g., disinformation, generating fake profiles, surveillance), fairness considerations (e.g., deployment of technologies that could make decisions that unfairly impact specific groups), privacy considerations, and security considerations.
        \item The conference expects that many papers will be foundational research and not tied to particular applications, let alone deployments. However, if there is a direct path to any negative applications, the authors should point it out. For example, it is legitimate to point out that an improvement in the quality of generative models could be used to generate Deepfakes for disinformation. On the other hand, it is not needed to point out that a generic algorithm for optimizing neural networks could enable people to train models that generate Deepfakes faster.
        \item The authors should consider possible harms that could arise when the technology is being used as intended and functioning correctly, harms that could arise when the technology is being used as intended but gives incorrect results, and harms following from (intentional or unintentional) misuse of the technology.
        \item If there are negative societal impacts, the authors could also discuss possible mitigation strategies (e.g., gated release of models, providing defenses in addition to attacks, mechanisms for monitoring misuse, mechanisms to monitor how a system learns from feedback over time, improving the efficiency and accessibility of ML).
    \end{itemize}
    
\item {\bf Safeguards}
    \item[] Question: Does the paper describe safeguards that have been put in place for responsible release of data or models that have a high risk for misuse (e.g., pre-trained language models, image generators, or scraped datasets)?
    \item[] Answer: \answerNA{} 
    \item[] Justification: The paper poses no such risks.
    \item[] Guidelines:
    \begin{itemize}
        \item The answer \answerNA{} means that the paper poses no such risks.
        \item Released models that have a high risk for misuse or dual-use should be released with necessary safeguards to allow for controlled use of the model, for example by requiring that users adhere to usage guidelines or restrictions to access the model or implementing safety filters. 
        \item Datasets that have been scraped from the Internet could pose safety risks. The authors should describe how they avoided releasing unsafe images.
        \item We recognize that providing effective safeguards is challenging, and many papers do not require this, but we encourage authors to take this into account and make a best faith effort.
    \end{itemize}

\item {\bf Licenses for existing assets}
    \item[] Question: Are the creators or original owners of assets (e.g., code, data, models), used in the paper, properly credited and are the license and terms of use explicitly mentioned and properly respected?
    \item[] Answer: \answerYes{} 
    \item[] Justification: We have referenced the required papers.
    \item[] Guidelines:
    \begin{itemize}
        \item The answer \answerNA{} means that the paper does not use existing assets.
        \item The authors should cite the original paper that produced the code package or dataset.
        \item The authors should state which version of the asset is used and, if possible, include a URL.
        \item The name of the license (e.g., CC-BY 4.0) should be included for each asset.
        \item For scraped data from a particular source (e.g., website), the copyright and terms of service of that source should be provided.
        \item If assets are released, the license, copyright information, and terms of use in the package should be provided. For popular datasets, \url{paperswithcode.com/datasets} has curated licenses for some datasets. Their licensing guide can help determine the license of a dataset.
        \item For existing datasets that are re-packaged, both the original license and the license of the derived asset (if it has changed) should be provided.
        \item If this information is not available online, the authors are encouraged to reach out to the asset's creators.
    \end{itemize}

\item {\bf New assets}
    \item[] Question: Are new assets introduced in the paper well documented and is the documentation provided alongside the assets?
    \item[] Answer: \answerNA{} 
    \item[] Justification: {There are no new assets.}
    \item[] Guidelines:
    \begin{itemize}
        \item The answer \answerNA{} means that the paper does not release new assets.
        \item Researchers should communicate the details of the dataset\slash code\slash model as part of their submissions via structured templates. This includes details about training, license, limitations, etc. 
        \item The paper should discuss whether and how consent was obtained from people whose asset is used.
        \item At submission time, remember to anonymize your assets (if applicable). You can either create an anonymized URL or include an anonymized zip file.
    \end{itemize}

\item {\bf Crowdsourcing and research with human subjects}
    \item[] Question: For crowdsourcing experiments and research with human subjects, does the paper include the full text of instructions given to participants and screenshots, if applicable, as well as details about compensation (if any)? 
    \item[] Answer: \answerNA{} 
    \item[] Justification: The paper does not involve crowdsourcing.
    \item[] Guidelines: 
    \begin{itemize}
        \item The answer \answerNA{} means that the paper does not involve crowdsourcing nor research with human subjects.
        \item Including this information in the supplemental material is fine, but if the main contribution of the paper involves human subjects, then as much detail as possible should be included in the main paper. 
        \item According to the NeurIPS Code of Ethics, workers involved in data collection, curation, or other labor should be paid at least the minimum wage in the country of the data collector. 
    \end{itemize}

\item {\bf Institutional review board (IRB) approvals or equivalent for research with human subjects}
    \item[] Question: Does the paper describe potential risks incurred by study participants, whether such risks were disclosed to the subjects, and whether Institutional Review Board (IRB) approvals (or an equivalent approval/review based on the requirements of your country or institution) were obtained?
    \item[] Answer: \answerNA{} 
    \item[] Justification: The paper does not involve crowdsourcing.
    \item[] Guidelines:
    \begin{itemize}
        \item The answer \answerNA{} means that the paper does not involve crowdsourcing nor research with human subjects.
        \item Depending on the country in which research is conducted, IRB approval (or equivalent) may be required for any human subjects research. If you obtained IRB approval, you should clearly state this in the paper. 
        \item We recognize that the procedures for this may vary significantly between institutions and locations, and we expect authors to adhere to the NeurIPS Code of Ethics and the guidelines for their institution. 
        \item For initial submissions, do not include any information that would break anonymity (if applicable), such as the institution conducting the review.
    \end{itemize}

\item {\bf Declaration of LLM usage}
    \item[] Question: Does the paper describe the usage of LLMs if it is an important, original, or non-standard component of the core methods in this research? Note that if the LLM is used only for writing, editing, or formatting purposes and does \emph{not} impact the core methodology, scientific rigor, or originality of the research, declaration is not required.
    \item[] Answer: \answerYes{} 
    \item[] Justification: {Used for writing and editing etc.}
    \item[] Guidelines:
    \begin{itemize}
        \item The answer \answerNA{} means that the core method development in this research does not involve LLMs as any important, original, or non-standard components.
        \item Please refer to our LLM policy in the NeurIPS handbook for what should or should not be described.
    \end{itemize}

\end{enumerate}

\newpage
\appendix
\section*{Appendix}
\section{Theoretical Analysis}

\subsection{Notation for Proofs}
\label{app:notation}

We work with the two-mode Gaussian mixture of Assumption~(A3).
Let $\mu_1,\mu_2\in\Mc$ be the two mode centres with separation
$\delta_c:=\|\mu_1-\mu_2\|>0$, unit intermode axis
$u:=(\mu_2-\mu_1)/\delta_c$, and midpoint $\bar\mu:=(\mu_1+\mu_2)/2$.
Adopting centred coordinates, $\langle\mu_1,u\rangle=-\delta_c/2$
and $\langle\mu_2,u\rangle=\delta_c/2$.
The within-mode covariance is $\Sigma_c^{\mathrm{in}}:=P_u^\perp+\sigma_u^2 P_u$,
where $P_u:=uu^\top$ and $P_u^\perp:=\mathbf{I}_{d_c}-P_u$, so
$\Sigma_c^{\mathrm{in}}$ has eigenvalue $\sigma_u^2$ along $u$ and
eigenvalue $1$ in the remaining $d_c-1$ tangential directions.
The total marginal covariance of the equal-weight mixture is
$\Sigma_c:=P_u^\perp+\Lambda_c P_u$ with
$\Lambda_c:=\sigma_u^2+\delta_c^2/4$, so $\operatorname{tr}\Sigma_c=(d_c-1)+\Lambda_c$
and the eigenvalues of $\Sigma_c$ are $1$ (multiplicity $d_c-1$) and
$\Lambda_c$ (multiplicity 1).

We use the following quantities throughout the proofs.
$V_\tau^u:=(1-\abar_\tau)+\abar_\tau\sigma_u^2$ is the marginal variance
of $x_{\tau,u}:=\langle x_\tau,u\rangle$ given either mode, arising
from $x_\tau=\sqrt{\abar_\tau}x_0+\sqrt{1-\abar_\tau}\varepsilon$
with $\varepsilon\sim\mathcal{N}(0,\mathbf{I})$.
$\Phi$ denotes the standard Gaussian CDF.
$M_t:=\widehat{x}_0(x_t,t;c)=\mathbb{E}[x_0\mid x_t,c]$ is the
Tweedie posterior mean.
$\pi_k(x_\tau):=\Pr[\text{mode }k\mid x_\tau]$ is the posterior mode
probability.
$\Pi_{\Mc}(x)\in\{\mu_1,\mu_2\}$ denotes nearest-mode projection.
In centred coordinates $\|x-\mu_1\|^2-\|x-\mu_2\|^2=2\delta_c\langle x,u\rangle$,
so $\Pi_{\Mc}(x)=\mu_2$ if $\langle x,u\rangle\ge 0$ and
$\Pi_{\Mc}(x)=\mu_1$ otherwise (ties, which have probability zero, are
assigned to $\mu_2$).
$K\in\{1,2\}$ denotes the mode from which $x_0$ is drawn.

\subsection{Standing Assumptions}
\label{app:assump}

\textbf{(A1) Network regularity.}
Each block $f_\ell$ is $C^2$ along the sampling trajectory with
$\|J_\ell\|_{\mathrm{op}}\le K_\ell$ and
$\|\nabla^2 h_\ell\|_{\mathrm{op}}\le M_\ell$ a.s., where
$J_\ell:=\nabla_x h_\ell$.

\textbf{(A2) Manifold regularity.}
$\Mc$ is a closed $C^2$ submanifold of $\Xc$ with positive reach
$\rho_c>0$ (guaranteed by Definition~\ref{def:CMDM} when $g$ is
$C^2$ with the submersion condition holding).

\textbf{(A3) Two-mode Gaussian mixture.}
$p(x_0\mid c)=\tfrac{1}{2}\mathcal{N}(\mu_1,\Sigma_c^{\mathrm{in}})+
\tfrac{1}{2}\mathcal{N}(\mu_2,\Sigma_c^{\mathrm{in}})$
(notation as in \S\ref{app:notation}), with the following conditions:
(i) $0<\sigma_u\le 1$;
(ii) $\Lambda_c>1$ (equivalently $\delta_c>2\sqrt{1-\sigma_u^2}$),
ensuring all eigenvalues of $\Sigma_c$ satisfy $\lambda_i\ge 1$;
(iii) $\delta_c>0$.

\textbf{(A4) ODE remainder.}
$\mathbb{E}\|R(\Delta t)\|^4\le C_R\Delta t^8$ for the Euler
remainder $R(\Delta t)$ of the probability-flow ODE (the local error of
an Euler step is $O(\Delta t^2)$).

\textbf{(A5) Directional Lipschitz constants and depth ordering.}
For each block $\ell$ and $t\le T^*$, there exist $L_\ell^\perp,
L_\ell^\parallel>0$ with $\|J_\ell P^\perp_x v\|\le L_\ell^\perp\|v\|$
for $v\in N_{\Pi(x)}\Mc$ and $\|J_\ell P^\parallel_x w\|\le L_\ell^\parallel\|w\|$
for $w\in T_{\Pi(x)}\Mc$.
Under unit-gain normalisation $L_\ell^\perp+L_\ell^\parallel=1$, set
$\beta_\ell:=L_\ell^\perp$ and $L_\ell^\parallel=1-\beta_\ell$.
The depth-ordering condition $\beta_{\ell_1}<\beta_{\ell_2}$ for
$\ell_1<\ell_2$ holds throughout the safe phase.

\textbf{(A5$'$) Transversality} (for Theorem~\ref{thm:depth}, two-sided
bound only).
There exist $\theta_*\in(0,\pi/2]$, $\zeta\in(0,\cos\theta_*]$, and
$\gamma_*>0$ such that for all non-zero $J_\ell v$, $J_\ell w$ with
$v\in N_{\Pi(x)}\Mc$, $w\in T_{\Pi(x)}\Mc$:
$\langle J_\ell v,J_\ell w\rangle/(\|J_\ell v\|\|J_\ell w\|)\ge-\cos\theta_*+\zeta$,
so $\alpha:=\cos\theta_*-\zeta\ge 0$.
Additionally: $\|J_\ell v\|\ge\underline{L}_\ell^\perp\|v\|$ and
$\|J_\ell w\|\ge\underline{L}_\ell^\parallel\|w\|$, with
$\gamma_*:=\min(\underline{L}_\ell^\perp/L_\ell^\perp,\,\underline{L}_\ell^\parallel/L_\ell^\parallel)>0$.
The constraint $\zeta\le\cos\theta_*$ (i.e., $\alpha\ge 0$) is necessary
for the AM-GM step in the lower bound (Step~6).

\textbf{(A6) Temporal regularity.}
$\mathbb{E}\,\delta_\ell(t_1)^2\le\eta\,\mathbb{E}\,\delta_\ell(t_2)^2$
for all $t_1,t_2$ within any cache stride window, for some $\eta\ge 1$.

\subsection{Lemma 1: Exact Posterior Variance}
\label{lem:posterior-var-sec}

\begin{lemma}[MMSE upper bound]
\label{lem:posterior-var}
For any prior $p(x_0\mid c)$ with mean zero and covariance $\Sigma_c$
having eigenvalues $\{\lambda_i\}_{i=1}^{d_c}$, and the VP forward
kernel $q(x_\tau\mid x_0)=\mathcal{N}(\sqrt{\abar_\tau}x_0,(1-\abar_\tau)\mathbf{I})$,
the Tweedie estimate $M_\tau=\mathbb{E}[x_0\mid x_\tau]$ satisfies
\[
  \mathbb{E}\|M_\tau-x_0\|^2
  \;\le\;\sum_{i=1}^{d_c}
  \frac{(1-\abar_\tau)\lambda_i}{\abar_\tau\lambda_i+1-\abar_\tau}.
\]
\end{lemma}

\begin{proof}
The nonlinear MMSE estimator $M_\tau=\mathbb{E}[x_0\mid x_\tau]$
achieves lower mean-squared error than any linear estimator.
Work in the eigenbasis of $\Sigma_c$.
The linear MMSE estimator of $x_0^{(i)}$ from
$x_\tau^{(i)}=\sqrt{\abar_\tau}x_0^{(i)}+\sqrt{1-\abar_\tau}\varepsilon^{(i)}$
has MSE
$\mathrm{Cov}(x_0^{(i)})-\mathrm{Cov}(x_0^{(i)},x_\tau^{(i)})^2/\mathrm{Var}(x_\tau^{(i)})
=\lambda_i-\abar_\tau\lambda_i^2/(\abar_\tau\lambda_i+1-\abar_\tau)
=(1-\abar_\tau)\lambda_i/(\abar_\tau\lambda_i+1-\abar_\tau)$,
using $\mathrm{Cov}(x_0^{(i)},x_\tau^{(i)})=\sqrt{\abar_\tau}\lambda_i$ and
$\mathrm{Var}(x_\tau^{(i)})=\abar_\tau\lambda_i+1-\abar_\tau$.
Since $\mathbb{E}\|M_\tau-x_0\|^2\le\mathbb{E}\|\hat{x}_{0,\mathrm{lin}}-x_0\|^2
=\sum_i(1-\abar_\tau)\lambda_i/(\abar_\tau\lambda_i+1-\abar_\tau)$, the result follows.
\end{proof}

\subsection{Proof of Theorem~\ref{thm:threshold}}
\label{sec:proof-t1}

\textbf{Setup.}
Freezing the sampler at cache time $\tau$ yields output
$x_0^\tau=M_\tau:=\widehat{x}_0(x_\tau,\tau;c)$.
Since $M_t=\mathbb{E}[x_0\mid x_t,c]$ is a reverse-time martingale
(by the tower property: $\mathbb{E}[M_s\mid x_t]=\mathbb{E}[x_0\mid x_t]=M_t$
for $s<t$) with $M_0=x_0$ a.s., the cache error is
$\mathcal{E}(\tau)=\mathbb{E}\|M_\tau-x_0\|^2$.

\textbf{Part (i): Safe regime $\abar_\tau\ge\tfrac{1}{2}$.}
By Lemma~\ref{lem:posterior-var} (MMSE $\le$ linear MMSE, valid for any prior with covariance $\Sigma_c$):
$\mathcal{E}(\tau)\le\sum_i(1-\abar_\tau)\lambda_i/(\abar_\tau\lambda_i+1-\abar_\tau)$.
Since $\lambda_i\ge 1$ (A3.ii), we have
$\abar_\tau\lambda_i+1-\abar_\tau\ge 1$, so each term is bounded by
$(1-\abar_\tau)\lambda_i$, giving
\begin{equation}
  \mathcal{E}(\tau)\;\le\;(1-\abar_\tau)\trS\;\le\;\tfrac{1}{2}\trS,
  \label{eq:safe-err}
\end{equation}
where the last step uses $1-\abar_\tau\le\tfrac{1}{2}$.

\textbf{Part (ii): Mode-flip probability.}
Fix $\tau\ge 1$, so that $\abar_\tau\in(0,1)$.
In centred coordinates ($\mu_1=-\frac{\delta_c}{2}u$,
$\mu_2=\frac{\delta_c}{2}u$), write $x_{\tau,u}:=\langle x_\tau,u\rangle$.
The marginal distribution of $x_{\tau,u}$ given mode $k$ is
$\mathcal{N}(\sqrt{\abar_\tau}\langle\mu_k,u\rangle, V_\tau^u)$.

\emph{Step 1: Log-posterior ratio.}
Under (A3), $\Sigma_c^{\mathrm{in}}$ is block-diagonal with respect to
$u$ and its orthogonal complement, the component of $x_0$ orthogonal to
$u$ has the same distribution under both modes, and the forward noise
is isotropic. Hence the two mode likelihoods differ only through
$x_{\tau,u}$. Since the two mode-conditional marginals of $x_{\tau,u}$
are Gaussian with the same variance $V_\tau^u$ and means differing by
$\sqrt{\abar_\tau}\delta_c$:
\[
  \log\frac{\pi_2(x_\tau)}{\pi_1(x_\tau)}
  \;=\;\log\frac{p(x_{\tau,u}\mid\text{mode }2)}{p(x_{\tau,u}\mid\text{mode }1)}
  \;=\;\frac{\sqrt{\abar_\tau}\delta_c}{V_\tau^u}\,x_{\tau,u},
\]
so $\pi_2\ge\tfrac{1}{2}$ if and only if $x_{\tau,u}\ge 0$.

\emph{Step 2: Tweedie flip (Lemma 2).}

\begin{lemma}[Tweedie mode flip]
\label{lem:tweedie-flip}
$\langle M_\tau,u\rangle\ge 0$ if and only if $x_{\tau,u}\ge 0$.
Hence $\Pi_{\Mc}(M_\tau)=\mu_2$ if $x_{\tau,u}\ge 0$ and
$\Pi_{\Mc}(M_\tau)=\mu_1$ otherwise.
\end{lemma}
\begin{proof}
The Tweedie formula gives $M_\tau=\pi_1 m_1+\pi_2 m_2$, where
$m_k:=\mathbb{E}[x_0\mid x_\tau,\text{mode }k]$.
By the independence noted in Step~1,
$\langle m_k,u\rangle=\mathbb{E}[\langle x_0,u\rangle\mid x_{\tau,u},\text{mode }k]$,
and the Gaussian regression formula
($\operatorname{Cov}(x_{0,u},x_{\tau,u})=\sqrt{\abar_\tau}\sigma_u^2$,
$\operatorname{Var}(x_{\tau,u})=V_\tau^u$) gives
\[
  \langle m_k,u\rangle
  \;=\;\langle\mu_k,u\rangle
    +\frac{\sqrt{\abar_\tau}\sigma_u^2}{V_\tau^u}
    (x_{\tau,u}-\sqrt{\abar_\tau}\langle\mu_k,u\rangle)
  \;=\;\frac{\sqrt{\abar_\tau}\sigma_u^2\,x_{\tau,u}
             +(1-\abar_\tau)\langle\mu_k,u\rangle}{V_\tau^u}.
\]
Hence $\langle m_2-m_1,u\rangle=(1-\abar_\tau)\delta_c/V_\tau^u>0$,
$\langle m_1,u\rangle=[\sqrt{\abar_\tau}\sigma_u^2\,x_{\tau,u}
-(1-\abar_\tau)\delta_c/2]/V_\tau^u$, and
$\langle m_2,u\rangle=[\sqrt{\abar_\tau}\sigma_u^2\,x_{\tau,u}
+(1-\abar_\tau)\delta_c/2]/V_\tau^u$.
On $\{x_{\tau,u}\ge 0\}$, we have $\pi_2\ge\tfrac{1}{2}$, so:
\begin{align*}
  \langle M_\tau,u\rangle
  &=\langle m_1,u\rangle+\pi_2\langle m_2-m_1,u\rangle\\
  &\ge\frac{\sqrt{\abar_\tau}\sigma_u^2\,x_{\tau,u}
            -(1-\abar_\tau)\delta_c/2}{V_\tau^u}
   +\frac{(1-\abar_\tau)\delta_c}{2V_\tau^u}
  \;=\;\frac{\sqrt{\abar_\tau}\sigma_u^2\,x_{\tau,u}}{V_\tau^u}\;\ge\;0.
\end{align*}
On $\{x_{\tau,u}<0\}$, we have $\pi_1>\tfrac{1}{2}$, so:
\begin{align*}
  \langle M_\tau,u\rangle
  &=\langle m_2,u\rangle-\pi_1\langle m_2-m_1,u\rangle\\
  &<\frac{\sqrt{\abar_\tau}\sigma_u^2\,x_{\tau,u}
            +(1-\abar_\tau)\delta_c/2}{V_\tau^u}
   -\frac{(1-\abar_\tau)\delta_c}{2V_\tau^u}
  \;=\;\frac{\sqrt{\abar_\tau}\sigma_u^2\,x_{\tau,u}}{V_\tau^u}\;<\;0.
\end{align*}
The last statement follows from the characterisation of $\Pi_{\Mc}$
in \S\ref{app:notation}.
\end{proof}

\emph{Step 3: Closed form.}
By Lemma~\ref{lem:tweedie-flip}, the frozen estimate is closer to the
wrong mode exactly when $x_{\tau,u}$ falls on the wrong side of zero.
Given mode 1, $x_{\tau,u}\sim\mathcal{N}(-\sqrt{\abar_\tau}\delta_c/2,\,V_\tau^u)$, so
\[
  \Pr[x_{\tau,u}\ge 0\mid K=1]
  \;=\;\Phi\Bigl(-\frac{\sqrt{\abar_\tau}\,\delta_c}{2\sqrt{V_\tau^u}}\Bigr),
\]
and by symmetry $\Pr[x_{\tau,u}<0\mid K=2]$ takes the same value, which
proves~\eqref{eq:pflip}.
Since $V_\tau^u=1-\abar_\tau(1-\sigma_u^2)$, the derivative of
$\abar_\tau/V_\tau^u$ with respect to $\abar_\tau$ is $(V_\tau^u)^{-2}>0$,
so $p_{\mathrm{flip}}$ is continuous and strictly decreasing in
$\abar_\tau$.

\emph{Safe regime ($\abar_\tau\ge\tfrac{1}{2}$).}
Since $\sigma_u\le 1$ (A3.i), $V_\tau^u\le 1$, hence
\[
  \sqrt{\abar_\tau/V_\tau^u}\;\ge\;\sqrt{\abar_\tau}\;\ge\;1/\sqrt{2}
  \quad\text{and}\quad
  p_{\mathrm{flip}}(\tau)\;\le\;\Phi(-\delta_c\sqrt{\abar_\tau}/2)
  \;\le\;\Phi(-\delta_c/(2\sqrt{2})).
\]

\emph{Unsafe regime ($\abar_\tau<\tfrac{1}{2}$).}
Here $V_\tau^u-\abar_\tau=1-2\abar_\tau+\abar_\tau\sigma_u^2>0$, so
$\sqrt{\abar_\tau/V_\tau^u}<1$ and
$p_{\mathrm{flip}}(\tau)>\Phi(-\delta_c/2)=p_{\mathrm{mix}}$.
This bound does not depend on $\abar_\tau$, so it does not vanish as
$\abar_\tau\uparrow\tfrac{1}{2}$.

\textbf{Part (iii): Role of $T^*$.}
Since $\Phi$ is strictly increasing,
$p_{\mathrm{flip}}(\tau)\ge p_{\mathrm{mix}}$ if and only if
$\abar_\tau\le V_\tau^u$, that is, $\abar_\tau(2-\sigma_u^2)\le 1$, or
$\abar_\tau\le 1/(2-\sigma_u^2)$.
For $\sigma_u\in(0,1]$ we have $1/(2-\sigma_u^2)>\tfrac{1}{2}$, and
$1/(2-\sigma_u^2)\to\tfrac{1}{2}$ as $\sigma_u\to 0$.
Hence the condition holds for every $\sigma_u\in(0,1]$ if and only if
$\abar_\tau\le\tfrac{1}{2}$. \hfill$\square$

\subsection{Proof of Theorem~\ref{thm:depth}}
\label{sec:proof-t2}

\textbf{Step 1: Linearise block drift.}
By Taylor's theorem and (A1):
$h_\ell(x_{t'},t,c)-h_\ell(x_t,t,c)=J_\ell\Delta x+E_\ell$,
with $\Delta x:=x_{t'}-x_t$ and $\|E_\ell\|\le\tfrac{M_\ell}{2}\|\Delta x\|^2=O(\Delta t^2)$.
Squaring, taking expectations, and bounding the cross term via
Cauchy--Schwarz (which gives $O(\Delta t^3)$):
\begin{equation}
  \mathbb{E}\,\delta_\ell(t)^2=\mathbb{E}\|J_\ell\Delta x\|^2+O(\Delta t^3).
  \label{eq:lin}
\end{equation}

\textbf{Step 2: Decompose the Euler step.}
The probability-flow ODE advances as
$x_{t'}=x_t+\Delta t\,A_t s_t(x_t;c)+R(\Delta t)$,
so $\Delta x=\Delta t A_t s_t+R(\Delta t)$.
Using $s_t=s_t^\perp+s_t^\parallel$ and the projectors:
$P^\perp_x\Delta x=\Delta t A_t s_t^\perp+P^\perp_x R$ and
$P^\parallel_x\Delta x=\Delta t A_t s_t^\parallel+P^\parallel_x R$.
With $\Theta_t:=A_t^2\mathbb{E}\|s_t^\parallel\|^2>0$ and
$\kappa_t:=\mathbb{E}\|s_t^\perp\|^2/\mathbb{E}\|s_t^\parallel\|^2$,
Assumption~(A4) gives $\mathbb{E}\|P^\perp_x R\|^2,\mathbb{E}\|P^\parallel_x R\|^2=O(\Delta t^4)$, so:
\begin{equation}
  \mathbb{E}\|P^\perp_x\Delta x\|^2=\Delta t^2\Theta_t\kappa_t+O(\Delta t^3),\quad
  \mathbb{E}\|P^\parallel_x\Delta x\|^2=\Delta t^2\Theta_t+O(\Delta t^3).
  \label{eq:decomp}
\end{equation}

\textbf{Step 3: Upper bound on drift.}
By (A5) and the triangle inequality:
$\|J_\ell\Delta x\|\le L_\ell^\perp\|P^\perp_x\Delta x\|+L_\ell^\parallel\|P^\parallel_x\Delta x\|$.
Using $(a+b)^2\le 2(a^2+b^2)$, taking expectations, and
substituting~\eqref{eq:decomp} into~\eqref{eq:lin}:
\begin{equation}
  \mathbb{E}\,\delta_\ell(t)^2
  \;\le\;\overline{B}_\ell\,\Delta t^2\Theta_t+O(\Delta t^3),\quad
  \overline{B}_\ell:=2[(L_\ell^\perp)^2\kappa_t+(L_\ell^\parallel)^2]
  =2B(\beta_\ell;\kappa_t).
  \label{eq:Bupp}
\end{equation}

\textbf{Step 4: Certified strides.}
Let $\overline{D}_\ell:=\bigl(\overline{B}_\ell\Delta t^2\Theta_t+O(\Delta t^3)\bigr)^{1/2}
=\Delta t\sqrt{\overline{B}_\ell\Theta_t}\,\bigl(1+O(\Delta t)\bigr)$
be the square root of the right-hand side of~\eqref{eq:Bupp}.
For every step $t-k\Delta t$ in a cache window starting at $t$,
Jensen's inequality and (A6) give
$\mathbb{E}\,\delta_\ell(t-k\Delta t)\le\bigl(\mathbb{E}\,\delta_\ell(t-k\Delta t)^2\bigr)^{1/2}
\le\sqrt{\eta}\,\overline{D}_\ell$.
Hence every stride $S$ with $(S-1)\sqrt{\eta}\,\overline{D}_\ell\le\varepsilon$
is $\varepsilon$-cacheable, and the certified stride
$\widehat{S}_\ell:=1+\varepsilon/(\sqrt{\eta}\,\overline{D}_\ell)$
satisfies $\lfloor\widehat{S}_\ell\rfloor\le S^\star_\ell$, where
$S^\star_\ell$ is the largest $\varepsilon$-cacheable stride.

\textbf{Step 5: Depth-monotone comparison.}
For $\ell_1<\ell_2$ with $\beta_{\ell_1}<\beta_{\ell_2}$ (A5):
\[
  \frac{\widehat{S}_{\ell_2}-1}{\widehat{S}_{\ell_1}-1}
  =\frac{\overline{D}_{\ell_1}}{\overline{D}_{\ell_2}}
  =\sqrt{\frac{B(\beta_{\ell_1};\kappa_t)}{B(\beta_{\ell_2};\kappa_t)}}
  (1+O(\Delta t)).
\]
Its leading factor exceeds one iff $B(\beta_{\ell_1};\kappa_t)>B(\beta_{\ell_2};\kappa_t)$,
i.e., $(\beta_{\ell_1}^2-\beta_{\ell_2}^2)\kappa_t>(1-\beta_{\ell_2})^2-(1-\beta_{\ell_1})^2$.
Dividing both sides by $(\beta_{\ell_2}^2-\beta_{\ell_1}^2)>0$ and
negating:
$\kappa_t<\kappa^*:=[(1-\beta_{\ell_1})+(1-\beta_{\ell_2})]/(\beta_{\ell_1}+\beta_{\ell_2})$,
which is~\eqref{eq:kappa-threshold}.
This establishes the envelope comparison and the certified-stride
comparison of Theorem~\ref{thm:depth}. \hfill$\square$ (Theorem~2a)

\textbf{Step 6: Actual drift under transversality.}
Write $v:=P^\perp_x\Delta x$ and $w:=P^\parallel_x\Delta x$.
From (A5$'$) with $\alpha:=\cos\theta_*-\zeta\ge 0$:
$2\langle J_\ell v,J_\ell w\rangle\ge-2\alpha\|J_\ell v\|\|J_\ell w\|\ge-\alpha(\|J_\ell v\|^2+\|J_\ell w\|^2)$,
where the last step applies AM-GM and uses $\alpha\ge 0$.
Hence:
\[
  \|J_\ell\Delta x\|^2=\|J_\ell v\|^2+\|J_\ell w\|^2+2\langle J_\ell v,J_\ell w\rangle
  \;\ge\;(1-\alpha)(\|J_\ell v\|^2+\|J_\ell w\|^2).
\]
Applying the lower-Lipschitz bounds ($\|J_\ell v\|\ge\underline{L}_\ell^\perp\|v\|\ge\gamma_*\beta_\ell\|v\|$
and similarly for $w$):
\begin{equation}
  \mathbb{E}\,\delta_\ell^2(t)\;\ge\;\underline{B}_\ell\,\Delta t^2\Theta_t-O(\Delta t^3),\quad
  \underline{B}_\ell:=(1-\alpha)\gamma_*^2[\beta_\ell^2\kappa_t+(1-\beta_\ell)^2]
  =\tfrac{(1-\alpha)\gamma_*^2}{2}\overline{B}_\ell.
  \label{eq:Blow}
\end{equation}
Combining~\eqref{eq:Bupp} for $\ell_2$ with~\eqref{eq:Blow} for $\ell_1$,
\begin{align*}
  \mathbb{E}\,\delta_{\ell_2}(t)^2
  &\le 2B(\beta_{\ell_2};\kappa_t)\,\Delta t^2\Theta_t+O(\Delta t^3),\\
  \mathbb{E}\,\delta_{\ell_1}(t)^2
  &\ge (1-\alpha)\gamma_*^2B(\beta_{\ell_1};\kappa_t)\,\Delta t^2\Theta_t-O(\Delta t^3).
\end{align*}
If $B(\beta_{\ell_1};\kappa_t)/B(\beta_{\ell_2};\kappa_t)>2/\bigl((1-\alpha)\gamma_*^2\bigr)$,
the leading term of the first bound is strictly smaller than that of
the second, so $\mathbb{E}\,\delta_{\ell_2}(t)^2<\mathbb{E}\,\delta_{\ell_1}(t)^2$
for small enough $\Delta t$, completing Theorem~2b. \hfill$\square$

\subsection{Discussion}
\label{sec:discussion}

Theorem~\ref{thm:threshold} is proved under a two-mode Gaussian mixture
(A3), the simplest setting in which the mode-flip probability can be
computed in closed form. We use it as a simplified model: the data
distributions of the evaluated models are multimodal and non-Gaussian,
and we do not claim that they satisfy (A3).
Theorem~\ref{thm:threshold} is related to analyses of critical noise
levels at which the modes of a mixture separate during
sampling~\citep{raya2023spontaneous,biroli2024dynamical,li2024critical},
and more broadly to theoretical studies of diffusion models on simple
mixture distributions~\citep{shah2023learning,chidambaram2024guidance}.
Two parts of the analysis do not depend on (A3): the MMSE bound of
Lemma~\ref{lem:posterior-var} holds for any prior with covariance
$\Sigma_c$, and Theorem~\ref{thm:depth} does not use (A3).
Theorem~\ref{thm:threshold} also does not use the smoothness or the
reach of $\Mc$; only Theorem~\ref{thm:depth} needs the tubular
neighbourhood of (A2), and the runtime schedule never evaluates $g$.
Assumption~(A3.ii) ($\lambda_i\ge 1$) is a coordinate normalisation used
only for the simplified bound $\tfrac{1}{2}\trS$.

\textbf{Exact versus learned score.}
Both theorems are stated for the exact score.
For a learned score $\hat{s}_\tau$, the learned Tweedie estimate
$\hat{x}_0^\tau:=(x_\tau+(1-\abar_\tau)\hat{s}_\tau(x_\tau))/\sqrt{\abar_\tau}$
satisfies
$\hat{x}_0^\tau-x_0^\tau=\tfrac{1-\abar_\tau}{\sqrt{\abar_\tau}}\,(\hat{s}_\tau-s_\tau)$,
so by the triangle inequality
$\bigl(\mathbb{E}\|\hat{x}_0^\tau-x_0\|^2\bigr)^{1/2}\le\mathcal{E}(\tau)^{1/2}
+\tfrac{1-\abar_\tau}{\sqrt{\abar_\tau}}\bigl(\mathbb{E}\|\hat{s}_\tau-s_\tau\|^2\bigr)^{1/2}$.
For $\tau\le T^*$ the factor $(1-\abar_\tau)/\sqrt{\abar_\tau}$ is at
most $1/\sqrt{2}$, so score errors are damped rather than amplified in
the regime where we cache.

\textbf{Regime condition of Theorem~\ref{thm:depth}.}
Theorem~\ref{thm:depth} is conditional on the depth ordering (A5) and on
$\kappa_t<\kappa^*$~\eqref{eq:kappa-threshold}; the latter is a regime
condition with an explicit form that we assume rather than derive from
the noise schedule.

\textbf{Codimension.}
The safe-side bound scales with $\trS$, the total variance along the
$d_c=D-k$ tangent directions, so for a fixed variance per direction a
higher codimension $k$ gives a smaller bound. In practice, outputs with
many constraints, such as multi-view video, are also less tolerant of
small errors: a sub-pixel misalignment between views already lowers
PSNR and raises LPIPS. We believe this is why the conservative setting
$(2,3)$ worked best for SV4D~2.0.

\textbf{Other samplers.}
Extensions to flow-matching models require redefining $T^*$ via the
corresponding SNR parameterisation; for stochastic samplers, injected
noise adds $O(\sqrt{1-\abar_t})$ per cached step and is left to future work.

\section{CMDM Membership of the Evaluated Models}
\label{app:cmdm}

Table~\ref{tab:cmdm} lists, for each evaluated model, the constraint
encoded by $g(x;c)$ and how the model keeps its samples on $\Mc$:
statistically (learned from data), structurally (built into the
parameterisation or architecture), or through learned geometry.
As discussed in \S\ref{sec:constraint-framework}, this mechanism does
not affect how \textsc{ManifoldCache} is applied.

\begin{table}[h]
\centering
\caption{Constraint encoded by $g(x;c)$ and how each evaluated model
keeps its samples on $\Mc$.}
\label{tab:cmdm}
\small
\begin{tabular}{l p{6.4cm} l}
\toprule
\textbf{Model} & \textbf{Constraint encoded by $g(x;c)$} & \textbf{How it is maintained} \\
\midrule
Med-DDPM       & anatomical intensity and continuity & statistical \\
DiffCSP++      & space-group and Wyckoff symmetry & structural \\
Proteina       & Ramachandran validity and steric-clash avoidance & learned geometry \\
SV4D~2.0, Lyra & multi-view photometric consistency & statistical \\
DiT-MC         & bond-geometry validity & learned geometry \\
UniMoMo        & $\mathrm{E}(3)$-equivariant block-graph validity & structural \\
DemoDiff       & discrete motif-graph validity (one category per token) & structural \\
\bottomrule
\end{tabular}
\end{table}

\section{Quality--Speed Trade-off on All Eight Models}
\label{app:tradeoff}

Figure~\ref{fig:tradeoff-all} shows the quality--speed trade-off for
all eight models, complementing the Med-DDPM view in
Figure~\ref{fig:teaser}.

\begin{figure}[t]
  \centering
  \includegraphics[width=0.9\textwidth]{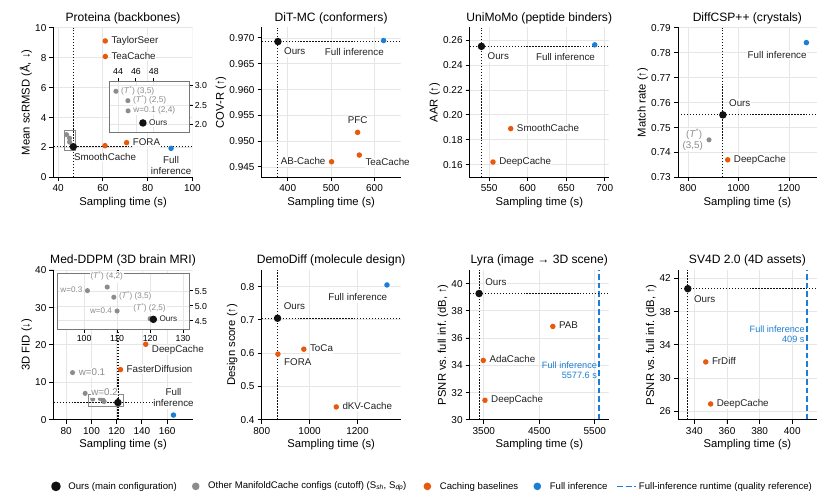}
  \caption{Quality--speed trade-off on all eight CMDMs. Each panel plots
  sampling time against the model's primary quality metric: mean scRMSD
  (Proteina), COV-R (DiT-MC), AAR (UniMoMo), Match Rate (DiffCSP++), FID
  (Med-DDPM), design Score (DemoDiff), and PSNR (Lyra, SV4D~2.0); the
  axis labels show whether higher or lower is better. ``Ours'' is the
  main \textsc{ManifoldCache} configuration in each table, the dotted
  lines pass through it, and grey points are the other
  \textsc{ManifoldCache} configurations and ablations. For Lyra and
  SV4D~2.0, quality is measured against full inference, so full
  inference appears only as a dashed line at its runtime. All values are
  taken from Tables~\ref{tab:runtime_scrmsd}--\ref{tab:viewpair_metrics}.}
  \label{fig:tradeoff-all}
\end{figure}

\section{Computational Cost}
\label{app:flops}

Table~\ref{tab:flops} reports inference FLOPs for full inference, for
the caching baseline with the fewest FLOPs, and for
\textsc{ManifoldCache}. \textsc{ManifoldCache} needs the fewest FLOPs on
five of the six models, up to $18.4\%$ fewer than the best baseline
(DemoDiff); on SV4D~2.0 it needs $1.6\%$ more.

\begin{table}[h]
\centering
\caption{Inference FLOPs (lower is better). ``Best baseline'' is the
caching baseline with the fewest FLOPs for each model.}
\label{tab:flops}
\small
\begin{tabular}{l c c c}
\toprule
\textbf{Model} & \textbf{Full Inference} & \textbf{Best baseline} & \textbf{\textsc{ManifoldCache}} \\
\midrule
Med-DDPM & $1.06\times10^{17}$  & $7.01\times10^{16}$  & $\mathbf{6.64\times10^{16}}$ \\
Proteina & $7.20\times10^{14}$  & $6.06\times10^{14}$  & $\mathbf{5.80\times10^{14}}$ \\
UniMoMo  & $3.24\times10^{14}$  & $2.04\times10^{14}$  & $\mathbf{2.03\times10^{14}}$ \\
SV4D~2.0 & $2.137\times10^{15}$ & $\mathbf{1.945\times10^{15}}$ & $1.977\times10^{15}$ \\
DiT-MC   & $1.81\times10^{14}$  & $1.65\times10^{14}$  & $\mathbf{1.46\times10^{14}}$ \\
DemoDiff & $1.52\times10^{16}$  & $1.241\times10^{16}$ & $\mathbf{1.013\times10^{16}}$ \\
\bottomrule
\end{tabular}
\end{table}

Table~\ref{tab:nfe} reports the number of function evaluations (NFE) on
two models from different domains, Med-DDPM (3D medical imaging) and
Proteina (protein design).
\textsc{ManifoldCache} needs the fewest NFE on both models: $142$ instead
of $250$ on Med-DDPM ($188$ for the best baseline) and $164$ instead of
$200$ on Proteina ($171$ for the best baseline).

\begin{table}[h]
\centering
\caption{Number of function evaluations (NFE; lower is better).
``Best baseline'' is the caching baseline with the fewest NFE for each
model.}
\label{tab:nfe}
\small
\begin{tabular}{l c c c}
\toprule
\textbf{Model} & \textbf{Full Inference} & \textbf{Best baseline} & \textbf{\textsc{ManifoldCache}} \\
\midrule
Med-DDPM & 250 & 188 & \textbf{142} \\
Proteina & 200 & 171 & \textbf{164} \\
\bottomrule
\end{tabular}
\end{table}

\section{Comparison with Speculative Sampling}
\label{app:spec}

Speculative methods and \textsc{ManifoldCache} reduce cost in different
ways. Speculative methods reduce the number of expensive model calls: a
cheap proposal drafts future states, and the target model then accepts
or rejects them, which keeps the samples exactly distributed as those of
the target model~\citep{debortoli2025accelerated,agrawal2025spiffy}.
\textsc{ManifoldCache} has no draft model, verifier, or rejection step.
It keeps every denoising step and lowers the cost of each one by reusing
block outputs according to a fixed schedule that depends only on the
noise level and the block depth.

This leads to three practical differences.
First, speculative methods need either a separate draft model, which is
costly to obtain for 3D MRI or protein generation, or additional
machinery. The training-free drafts of \citet{debortoli2025accelerated}
require a stochastic sampler (the authors note that their method does
not apply directly to deterministic samplers), and
Spiffy~\citep{agrawal2025spiffy} calibrates its draft graphs offline for
discrete diffusion language models. \citet{debortoli2025accelerated} also
note that, unlike large language models, diffusion models are not
necessarily memory-bound, so parallel verification may bring smaller
wall-clock gains. \textsc{ManifoldCache} needs no calibration and runs on
all eight backbones.
Second, the guarantees differ: speculative sampling preserves the output
distribution of the target model, whereas \textsc{ManifoldCache} bounds
the error of reuse and the probability of landing in the wrong mode
(Theorem~\ref{thm:threshold}).
Third, the two act on different parts of the trajectory: the analysis of
\citet{debortoli2025accelerated} indicates that the acceptance rate is
high at the beginning of denoising, which is the high-noise phase in
which Theorem~\ref{thm:threshold} advises against caching. In principle
the two could therefore be combined, with speculation reducing the
number of steps and caching reducing the cost of each remaining step; we
leave this to future work.

Table~\ref{tab:spec} compares both speculative methods with
\textsc{ManifoldCache} on Proteina. Neither method was designed for
continuous protein backbone generation, so both had to be adapted. In
our adaptations, drafts are verified token by token, which captures the
continuous dynamics poorly; drafts were therefore often rejected and
rolled back, both runs were slower than full inference, and the
resulting scRMSD was much higher. These numbers reflect our adaptations
rather than the performance of either method in the setting it was
designed for.

\begin{table}[h]
\centering
\caption{Speculative sampling and \textsc{ManifoldCache} on Proteina.
Time and scRMSD for full inference and \textsc{ManifoldCache} are taken
from Table~\ref{tab:runtime_scrmsd}.}
\label{tab:spec}
\small
\begin{tabular}{l c c c}
\toprule
\textbf{Method} & \textbf{Time (sec)} $\downarrow$ & \textbf{Mean scRMSD} $\downarrow$ & \textbf{FLOPs} $\downarrow$ \\
\midrule
Full Inference & 90.47 & 1.93 & $7.20\times10^{14}$ \\
\midrule
Speculative sampling~\citep{debortoli2025accelerated} & 281 & 16.2 & $2.8\times10^{15}$ \\
Spiffy~\citep{agrawal2025spiffy} & 377 & 18.2 & $3.8\times10^{15}$ \\
\midrule
\rowcolor{gray!15}
\textsc{ManifoldCache} $(T^*)\,(2,4)$ & \textbf{46.78} & \textbf{2.04} & $5.80\times10^{14}$ \\
\bottomrule
\end{tabular}
\end{table}

\section{Qualitative Results}

Figure~\ref{fig:lyra} shows that \textsc{ManifoldCache} achieves the best visual quality on Lyra, particularly in preserving fine details.

\begin{figure*}[t]
    \centering
    \includegraphics[width=1\textwidth]{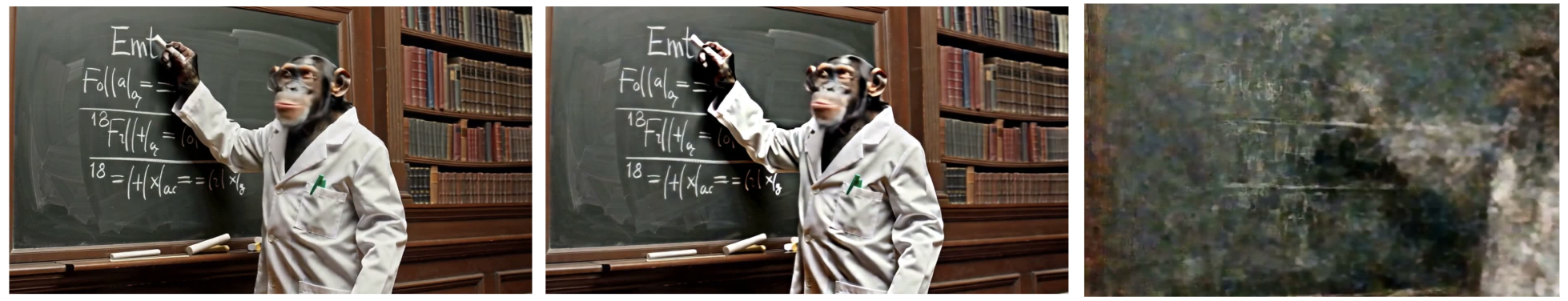}
    \caption{\small From left to right: \textsc{ManifoldCache}, PAB, and AdaCache. It can be seen that \textsc{ManifoldCache} achieves the best visual quality on Lyra.}
    \label{fig:lyra}
\end{figure*}

\section{Limitations and Future Work}
\label{limitt}
Despite strong empirical and theoretical performance, \textsc{ManifoldCache} has several important limitations that motivate future work. First, our theory relies on smooth closed manifold assumptions with well-defined tangent and normal decompositions. While these assumptions cover a broad range of scientific diffusion models, some structured generation problems involve non-smooth, discrete, combinatorial, or dynamically changing constraint sets where the local geometric decomposition may not exist or may vary discontinuously across the trajectory. \textbf{Extending manifold-aware caching to hybrid discrete--continuous diffusion processes, topology-changing manifolds, or non-Euclidean latent spaces remains an open direction}. Second, the theoretical guarantees are derived under simplified regularity and mixture assumptions designed to make the analysis tractable. Real CMDMs exhibit highly multimodal, non-Gaussian, and hierarchical distributions whose geometry may evolve across denoising stages. Although the ablations support the choice of $T^*$, \textbf{deriving sharper guarantees under more realistic score models, heavy-tailed distributions, or learned manifold geometries would strengthen the theoretical foundations of geometry-aware acceleration}. In addition, the depth ordering of Theorem~\ref{thm:depth} relies on the regime condition $\kappa_t<\kappa^*$, which we assume rather than derive. Third, the current schedule is analytically fixed from the noise schedule and block depth and does not adapt online to individual samples, prompts, or trajectory difficulty. Certain trajectories may approach the manifold substantially earlier or later than predicted by the global schedule, especially in highly multimodal settings. \textbf{Future work could investigate sample-adaptive or uncertainty-aware caching policies that dynamically estimate manifold proximity, score curvature, or feature drift during inference while retaining the low-overhead nature of the current framework}. Fourth, although \textsc{ManifoldCache} is backbone-agnostic, the present implementation focuses on inference-time feature reuse and does not jointly optimize with other acceleration mechanisms such as learned schedulers, consistency distillation, speculative denoising, low-rank adaptation, or mixed-precision execution. \textbf{Combining geometry-aware caching with solver acceleration, quantization, or token sparsification could potentially yield multiplicative speedups beyond those observed here}. Fifth, our experiments focus on single-device inference on existing pretrained checkpoints. Large-scale distributed generation systems introduce additional communication and synchronization bottlenecks that may alter the trade-off between recomputation and cache reuse. \textbf{Investigating geometry-aware caching under tensor parallelism, pipeline parallelism, or streaming generation settings is an important systems-level extension}. Finally, our work studies geometry primarily as a mechanism for safe acceleration. \textbf{A broader future direction is to treat manifold structure as a unifying principle connecting controllability, robustness, interpretability, and efficiency in scientific generative models}. We believe CMDMs provide a promising foundation for studying these properties jointly, particularly in emerging multimodal settings combining physical simulation, language conditioning, molecular reasoning, and embodied 3D generation.

\end{document}